\documentclass{article}

\PassOptionsToPackage{round, compress, authoryear}{natbib}

\usepackage[final]{neurips_2022}

\usepackage[utf8]{inputenc} 
\usepackage[T1]{fontenc}    
\usepackage{hyperref}       
\usepackage{url}            
\usepackage{booktabs}       
\usepackage{amsfonts}       
\usepackage{nicefrac}       
\usepackage{microtype}      
\usepackage{xcolor}         
\usepackage[english]{babel}
\usepackage{amsthm}
\usepackage{wrapfig}

\theoremstyle{definition}
\newtheorem{definition}{Definition}[section]

\title{Learning Generalized Policy Automata  \\ for Relational Stochastic Shortest Path Problems}

\author{%
 Rushang Karia, Rashmeet Kaur Nayyar, Siddharth Srivastava \\
  School of Computing and Augmented Intelligence\\
  Arizona State University\\
  Tempe, AZ, USA 85281 \\
  \texttt{Rushang.Karia,rmnayyar,siddharths@asu.edu} \\
}

\usepackage{amsmath}

\usepackage{amssymb}

\usepackage{algorithm}
\usepackage{algorithmic}

\usepackage{amsthm}
\usepackage[english]{babel}
\newtheorem{theorem}{Theorem}[section]
\newtheorem{corollary}{Corollary}[section]

\usepackage{layouts}

\usepackage{graphicx}

\usepackage{pdflscape}

\usepackage{layouts}

\begin{document}

\maketitle

\begin{abstract}
Several goal-oriented problems in the real-world can be naturally expressed as Stochastic Shortest Path Problems (SSPs). However, the computational complexity of solving SSPs makes finding solutions to even moderately sized problems intractable. Currently, existing state-of-the-art planners and heuristics often fail to exploit knowledge learned from solving other instances. This paper presents an approach for learning \emph{Generalized Policy Automata} (GPA): non-deterministic partial policies that can be used to catalyze the solution process. GPAs are learned using relational, feature-based abstractions, which makes them applicable on broad classes of related problems with different object names and quantities. Theoretical analysis of this approach shows that it guarantees completeness and hierarchical optimality. Empirical analysis shows that this approach effectively learns broadly applicable policy knowledge in a few-shot fashion and significantly outperforms state-of-the-art SSP solvers on test problems whose object counts are far greater than those used during training.
\end{abstract}

\section{Introduction}

Goal-oriented Markov Decision Processes (MDPs) expressed as Stochastic Shortest Path problems (SSPs) have been the subject of active research since they provide a convenient framework for modeling the uncertainty in action execution that often arises in the real-world. Recently, research in deep learning has demonstrated success in solving goal-oriented MDPs using image-based state representations \citep{tamar2016value,pong2018temporal,levy2017learning}. However, such methods require significant human-engineering effort in finding transformations like grayscale conversion, etc., to yield representations that facilitate learning. Many practical problems however, are more intuitively expressed using relational representations and have been widely studied in the literature.

As an example, consider a planetary rover whose mission is to collect all rocks of interest from a planet's surface and deliver them to the base for analysis. Such a problem objective is not easily described in an image-based representation (e.g., visibility is affected by line of sight) but can be easily described using a relational description language such as first-order logic.
Finding suitable image-based representations for such problems would be counter-productive and difficult. Furthermore, image-based deep learning methods often require large amounts of training data and/or are unable to provide guarantees of completeness and/or convergence.

Many real-world problems can be readily expressed as SSPs using symbolic descriptions that can be solved in polynomial time in terms of the state space. SSP algorithms use a combination of pruning strategies (e.g., heuristics \citep{hansen2001lao}) that can eliminate large parts of the search space from consideration, thereby reducing the computational effort expended. 
In spite of such optimizations, a major hurdle is the ``curse-of-dimensionality'' since the state spaces grow exponentially as the total number of objects increases. The pruning strategies employed by these SSP solvers do not scale well because they do not use knowledge that could have been exploited from solving similar problems. Existing SSP solvers would have difficulty scaling to rover problems with many locations and/or rocks. One solution to this problem is to compute a simple \emph{generalized policy}: move the rover to the closest available location with an interesting rock, try loading the rock until it succeeds, navigate back to the base, unload it, and re-iterate this process until all the interesting rocks are at the base. This generalized policy can be used to solve any rover instance with larger numbers of objects sharing a similar goal objective.

Related work in Generalized Planning addresses the problem of computing generalized policies by learning reliable controllers for broad classes of problems \citep{DBLP:conf/aaai/SrivastavaIZ08,DBLP:conf/aips/BonetPG09,DBLP:conf/ijcai/AguasCJ16}.
More recently, deep learning based approaches have been demonstrated to successfully learn generalized policies \citep{DBLP:conf/aaai/ToyerTTX18,groshev2018learning,DBLP:conf/icml/GargBM20}. A key limitation of these approaches is the lack of any theoretical guarantees of finding a solution or optimality. In this paper, we show that such policies can be learned with guarantees of completeness and hierarchical optimality using solutions of very few, \emph{small} problems with few objects.

The primary contribution of this paper is a novel approach
for few-shot learning of \emph{Generalized Policy Automata} (GPAs)
using solutions of SSP instances with small object counts. GPAs are non-deterministic partial policies that represent generalized knowledge that can be applied to problems with different object names and larger object counts. 
This process uses logical feature-based abstractions to lift instance-specific information like object names and counts while preserving the relationships between objects in a way that can be used to express generalized knowledge. 
GPAs learned using our approach can be used to \emph{accelerate any model-based SSP solver} by pruning out large sets of actions in different, related, but larger SSPs.
We prove that our approach is complete and guarantees hierarchical optimality. 
Empirical analysis on a range of well-known benchmark domains shows that our approach few-shot learns GPAs using as few as 3 training problem instances and convincingly outperforms existing state-of-the-art SSP solvers and does so without compromising the quality of the solutions found.

The rest of this paper is organized as follows: The next section provides the necessary background. Sec.\,\ref{sec:our_approach} describes our approach for using example policies in conjunction with abstractions to learn GPAs and use them for solving SSPs. We present our experimental setup and discuss obtained results in Sec.\,\ref{sec:experiments}. Sec.\,\ref{sec:related_work} provides a description of related work in the area. Finally, Sec.\,\ref{sec:conclusions} states the conclusions that we draw upon from this work followed by a brief description of future work.

\section{Background}
\label{sec:background}
Our problem setting considers SSPs expressed in a symbolic description language such as the Probabilistic Planning Domain Definition Language (PPDDL)  \citep{DBLP:journals/jair/YounesLWA05}. Let $\mathcal{D} = \langle \mathcal{P}, \mathcal{A} \rangle$ be a problem \emph{domain} where $\mathcal{P}$ and $\mathcal{A}$ are finite sets of predicates and parameterized actions. Object \emph{types}, such as those used in PPDDL, can be equivalently represented using unary predicates. A relational SSP problem instance for a domain $\mathcal{D}$ with a goal formula $g$ over $\mathcal{P}$ and a finite set of objects $O$ is defined as a tuple $P = \langle O, S, A, s_0, g, T, C \rangle$. A fact is the instantiation of a predicate $p \in \mathcal{P}$ with the 
appropriate number of objects from $O$. A state $s$ is a set of true facts and the state space $S$ is defined as all possible sets of true facts derived using $\mathcal{D}$ and $O$. Similarly, the action space $A$ is instantiated using $\mathcal{A}$ and $O$. $T : S \times A \times S' \rightarrow [0, 1]$ is the transition function and $C: S \times A \times S \rightarrow \mathbb{R}^{+}$ is the cost function. An entry $t(s, a, s') \in T$ defines the probability of executing action $a$ in a state $s$ and ending up in a state $s'$ where $a \in A$, $s, s' \in S$, and $c(s, a, s') \in C$ indicates the cost incurred while doing so. Naturally,  $\sum_{s'}t(s, a, s') = 1$ for any $s \in S$ and $a \in A$. Note that $a$ refers to the instantiated action $a(o_1, \ldots, o_n)$, where $o_1, \ldots, o_n \in O$ are the action \emph{parameters}. We omit the parameters when it is clear from context. $s_0 \in S$ is a \emph{known} initial state. A goal state $s_g$ is a state s.t. $s_g \models g$. $c(s_g, a, s_g) = 0$ and $t(s_g, a, s_g) = 1$ for all such goal states for any action $a$. 
Additionally, termination (reaching a state s.t. $s \models g$) in an SSP is inevitable making the length of the horizon unknown but finite \citep{DBLP:books/lib/BertsekasT96}.


\textbf{Running example:} The planetary rover example can be expressed using a domain that consists of parameterized predicates \emph{connected}$(l_x, l_y)$, \emph{in-rover}$(r_x)$, \emph{rock-at}$(r_x, l_x)$, and actions load$(r_x, l_x)$, unload$(r_x, l_x)$, and move$(l_x,l_y)$. Object types can be denoted using unary predicates \emph{location}$(l_x)$ and \emph{rock}$(r_x)$. $l_x$, $l_y$, and $r_x$ are \emph{parameters} that can be instantiated with different locations and rocks, allowing an easy way to express different problems. Actions dynamics are described using closed-form probability distributions (e.g. loading a rock could be modeled so that the rover picks up the rock with a probability of 0.8) and this forms the transition function.  A simplified SSP problem that ignores connectivity and consists of two locations, a base location, and two rocks can be described using a set of objects $O = \{ l_1, l_2, \emph{base}, r_1, r_2 \}$. A state in this SSP $s_\emph{eg}$ that describes the situation where $r_2$ is being carried by the rover and $r_1$ is at $l_2$ can be written as 
$s_\emph{eg} = \{ 
    \emph{location}(l_1),
    \emph{location}(l_2),
    \emph{location}(\emph{base}),
    \emph{rock}(r_1),
    \emph{rock}(r_2),
    \emph{in-rover}(r_2),
    \emph{rock-at}(r_1, l_2)
\}$. The goal of delivering all the rocks to the base can be expressed as $\forall x \text{ } \emph{rock}(x) \implies \emph{rock-at}(x, \emph{base})$. Executing any action can be assumed to expend some fuel and as a result, the objective is to deliver all the rocks to the base in a way that minimizes the total fuel expended.


A solution to an SSP is a deterministic policy $\pi : S \rightarrow A$ which is a mapping from states to actions. A \emph{proper policy} is one that is well-defined for all states. A \emph{complete proper policy} is one for which the goal is guaranteed to be reachable from all possible states. By definition, SSPs must have at least one \emph{complete} proper policy \citep{DBLP:books/lib/BertsekasT96}. This can be overly limiting in practice since such a formulation does not model dead end states: states from which the goal is reachable with probability 0. A weaker formulation of an SSP stipulates that the goal must be reachable with a probability of 1 from $s_0$ i.e. whose solution is a \emph{partial} proper policy from $s_0$ that is defined for every reachable state from $s_0$. To use such a formulation, we focus on a broader class of relaxed SSPs called Generalized SSPs \citep{DBLP:conf/uai/KolobovMW12} that allow the presence of dead-end states and only require the existence of at least one \emph{partial} proper policy from $s_0$.
Henceforth, we use the term SSPs to refer to Generalized SSPs and focus only on \emph{partial} proper policies.

The value of a state $s$ when using a policy $\pi$ is the expected cost of executing $\pi(s)$ when starting in $s$ and following $\pi$ thereafter: $V^\pi(s) = \sum_{s' \in S}t(s, \pi(s), s')[c(s, \pi(s), s') +  V^\pi(s')]$ \citep{DBLP:books/lib/SuttonB98}.
$V$ is known as the value function. The optimal policy $\pi^*$ is a policy that is better than or equal to all other policies. $V^*$ is the optimal value function corresponding to $\pi^{*}$. $V^*$ and consequently $\pi^{*}$ can be computed by iteratively applying the \emph{Bellman optimality equations}:
\begin{align}
\label{eqn:bellman_optimality}
V^{*}(s) &= \min_{a \in A}  \sum_{s' \in S} t(s, a, s') [c(s, a, s') +  V^{*}(s')]
\end{align}
SSP solvers iteratively apply Eq.\,\ref{eqn:bellman_optimality} starting from $s_0$ to compute a policy, and under certain conditions, have been proved to converge to a policy that is $\epsilon$-consistent with $\pi^*$ \citep{hansen2001lao, DBLP:conf/aips/BonetG03}. 


Let $F_\alpha$ and $F_\beta$ be two sets of features. We use feature-based abstractions to lift problem-specific characteristics like object names and numbers in order to facilitate the learning of generalized knowledge that can be applied to problems irrespective of differences in such characteristics.
We define \emph{state abstraction} as a function $\alpha : F_\alpha, S_P \rightarrow \overline{S}$ that transforms the concrete state space $S_P$ for an SSP $P$ into a finite abstract state space $\overline{S}$. Similarly, \emph{action abstraction}
$\beta : F_\beta, S_P, A_P \rightarrow \overline{A}$ transforms the action space to a finite abstract action space. Typically, $|\overline{S}| \le |S_P|$ and $|\overline{A}| \le |A_P|$. We use $\overline{s} = \alpha(F_\alpha, s)$ and $\overline{a} = \beta(F_\beta, s, a)$ to represent abstractions of a concrete state $s$ and action $a$. In this paper, we utilize feature sets automatically derived using canonical abstraction \citep{DBLP:journals/toplas/SagivRW02} to compute such feature-based representations of $s$ and $a$. 
This is described in Sec.\,\ref{subsec:canonical_abstraction}. 

\section{Our Approach}
\label{sec:our_approach}

Our objective is to exploit knowledge from solutions of SSP instances with small object counts to learn Generalized Policy Automata (GPAs) that allow effective pruning of the search space of related SSPs with larger object counts. We accomplish this by using solutions to a small set of training instances that are easily solvable using existing SSP solvers, and using feature-based canonical abstractions to learn a GPA that encodes generalized partial policies and serves as a guide to prune the set of policies under consideration. We provide a brief description of canonical abstraction in Sec.\,\ref{subsec:canonical_abstraction},  define GPAs in
 Sec.\,\ref{subsec:gpga_automaton}, and describe our process to learn a GPA in Sec.\,\ref{subsec:learning_gpga}. We then describe our method (Alg.\,\ref{alg:convert_ssp}) for solving SSPs in Sec.\,\ref{subsec:solving_ssp}.

\begin{figure}[t]
\centering
\includegraphics[scale=0.45]{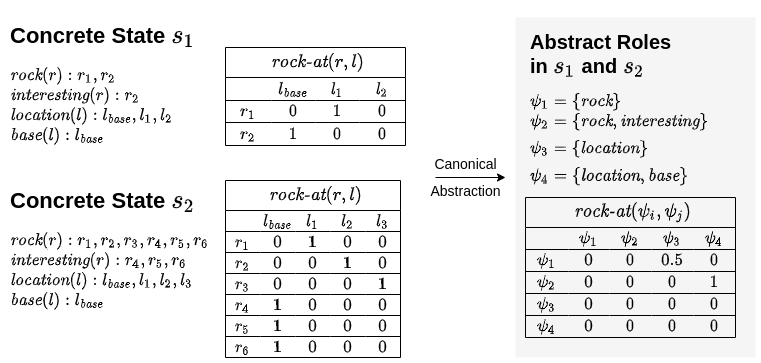}
\caption{An example of how canonical abstraction can be used to lift problem-specific characteristics like object names and numbers. $s_1$ and $s_2$ are example states of two \emph{different} problems.}
    \label{fig:abstraction_example}
\end{figure}

\subsection{Canonical Abstraction}
\label{subsec:canonical_abstraction}

Canonical abstractions, commonly used in program analysis, have been shown to be useful in generalized planning \citep{DBLP:journals/ai/SrivastavaIZ11,DBLP:conf/aaai/KariaS21} by allowing objects in a state to be grouped together into roles. Given a concrete state $s$ and an object $o$, the set of unary predicates that object $o$ satisfies is known as the \emph{role} of $o$.0-ary predicates are represented as unary predicates with a default ``phantom'' object. Multiple objects can map to the same role.

Let $\psi$ be a role, then, we define $\varphi_{\psi}(s)$ as a function that returns the set of objects that map to $\psi$ in a concrete state $s$. Similarly, for any given predicate $p_n \in \mathcal{P}$ where $n$ is the arity, $\varphi_{p_n}(\psi_{1}, \ldots, \psi_{n})$ is defined as the set of all $n$-ary predicates in $s$ that are consistent with the roles composing the predicate $p_{n}(\psi_{1}, \ldots, \psi_{n})$, i.e.,
$\varphi_{p_n(\psi_{1}, \ldots, \psi_{n})}(s) = \{ p_n(o_1, \ldots, o_n) | p_n(o_1, \ldots, o_n) \in s, o_i \in \varphi_{\psi_{i}(s)} \}$.

The value of a role $\psi$ in a concrete state $s$ is given as $\max(2, |\varphi_\psi(s)|)$ to indicate whether there are 0, 1, or greater than 1 objects satisfying the role. Since relations between objects become imprecise when grouped as roles, the value of a predicate $p_n(\psi_1, \ldots, \psi_n)$ in $s$ is determined using three-valued logic and is represented as $0$ if $\varphi_{p_n(\psi_{1}, \ldots, \psi_{n})}(s) = \{ \}$, as $1$ if $|\varphi_{p_n(\psi_{1}, \ldots, \psi_{n})}(s)| = |\varphi_{\psi_{1}}(s)| \times \ldots \times |\varphi_{\psi_{n}}(s)|$, and $\frac{1}{2}$ otherwise.

Let $\Psi$ be the set of all possible roles and $\mathcal{P}_i$ be the set of names of all predicates $p \in \mathcal{P}$ of arity $i$ for a domain $D$, then $\overline{\mathcal{P}}_i = \mathcal{P}_{i} \times [\Psi]^i$ is the set of all possible relations of arity $i$ between roles. We define the feature set for state abstraction as $F_\alpha = \Psi \cup_{i=2}^{N} \overline{\mathcal{P}}_i$ where $N$ is the maximum arity of any predicate in $D$. We define state abstraction $\alpha(F_\alpha, s)$ for a given concrete state $s$ to return an abstract state $\overline{s}$ as a total valuation of $F_\alpha$ using the process described above. Similarly, we define the feature set for action abstraction as $F_\beta = \Psi$. The action abstraction $\beta(F_\beta, s, a)$ for a concrete action $a(o_1, \ldots, o_n)$ when applied to $s$ returns an abstract action $\overline{a}(\psi_1, \ldots, \psi_n)$ where $\overline{a} \equiv a$ and $\psi_i$ is the role that object $o_i$ satisfies, i.e., $o_i \in \varphi_{\psi_{i}}(s)$ for $\psi_i \in \Psi$.

Fig.\,\ref{fig:abstraction_example} provides an intuitive example of how canonical abstraction can be used to lift instance-specific information like object quantities and object names. The figure describes two concrete states $s_1$ and $s_2$ from two different problems. The former contains 2 rocks and 3 locations whereas the latter contains 6 rocks and 4 locations. There are four distinct roles in $s_1$ and $s_2$. The abstract relation $\textit{rock-at}(\psi_i, \psi_j)$ provides the three-valued representation of values between different roles. For example, $\textit{rock-at}(\psi_1, \psi_3)$ interprets as \emph{the set of rocks that are at some location}, while $\textit{rock-at}(\psi_2, \psi_4)$ interprets as \emph{the set of interesting rocks that are at the base}.

It is easy to see that $r_1$ maps to the the role for a \emph{rock}, $\psi_1$, and  $r_2$ maps to the role for an \emph{interesting rock}, $\psi_2$. Since $\textit{rock-at}(r_1, l_2)$ does not appear in $s_1$, $\textit{rock-at}(\psi_1, \psi_3)$ evaluates to $0.5$ to indicate that there are \emph{some} objects of role \emph{rock} that are at some, but not all, locations. Similarly, since all \emph{interesting rocks} are at the base location, $\textit{rock-at}(\psi_2, \psi_4)$ evaluates to 1.

The key aspect of abstraction comes from the observation that the abstraction relation for $s_2$ remains the same even though $s_2$ has many more objects than $s_1$. The same high-level interpretations of the relations are captured while lifting low-level information like object names and numbers.

\subsection{Generalized Policy Automata}
\label{subsec:gpga_automaton}

We introduce Generalized Policy Automata (GPAs) as compact and expressive non-deterministic finite-state automata that encode generalized knowledge and can be represented as directed hypergraphs. GPAs impose hierarchical constraints on the state space of an SSP and prune the action space under consideration, thus reducing the computational effort of solving larger related SSP instances. 


\begin{wrapfigure}[18]{r}{0.62\textwidth}
\vspace{-2.5em}
\includegraphics[scale=0.55]{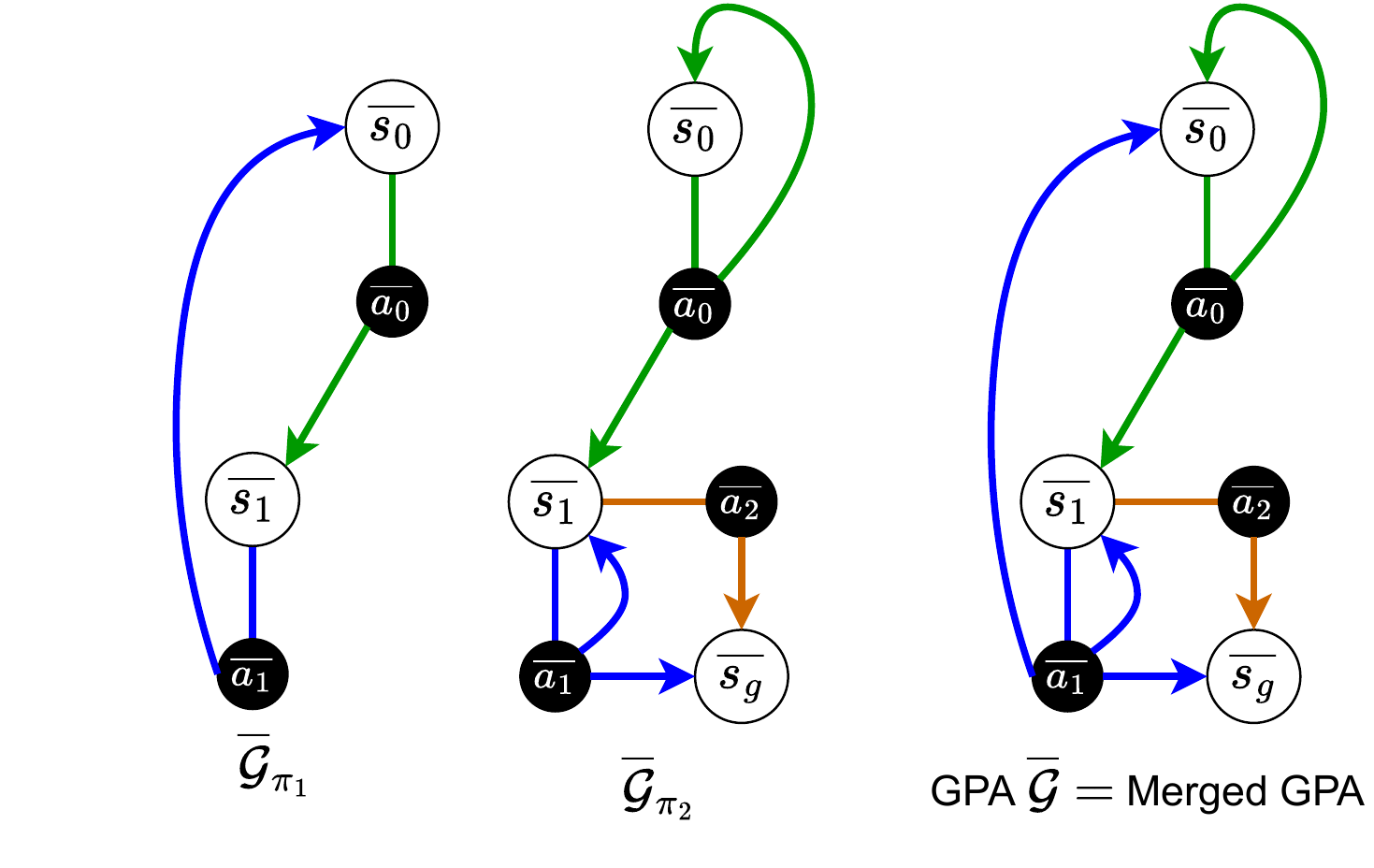}
\caption{A high-level overview of how we merge different GPAs. All edges with the same color represent a hyperedge. For example, the blue colored hyperedge in $\overline{\mathcal{G}}$ is $\langle \overline{s_1}, \{ \overline{s_0}, \overline{s_1}, \overline{s_g} \}, \overline{a_1} \rangle$.}
\label{fig:merge_gpas}

\end{wrapfigure}

\begin{definition}[Generalized Policy Automaton]

Let $\overline{S}$ and $\overline{A}$ be a set of \emph{abstract} states and actions.
A Generalized Policy Automaton (GPA) $\overline{\mathcal{G}} = \langle \overline{\mathcal{V}}, \overline{\mathcal{E}}\rangle$ is a non-deterministic finite-state automaton that can be represented as a directed hypergraph where $\overline{\mathcal{V}} = \overline{S}$and $\overline{\mathcal{E}} \subseteq \overline{\mathcal{V}} \times \mathbb{P}(\overline{\mathcal{V}}) \setminus \varnothing \times \overline{A}$ where $\mathbb{P}(\overline{\mathcal{V}})$ is the powerset of $\overline{\mathcal{V}}$ is the set of directed hyperedges s.t. each hyperedge $\overline{e} \in \overline{E}$ is a tuple $( \overline{e}_{\emph{src}}, \overline{e}_{\emph{dest}}, \overline{e}_{\emph{act}} )$ representing a start vertex, a set of result vertices, and an action label.
\end{definition}

\begin{definition}[GPA Consistent Policy]

A policy $\pi$ is defined to be consistent with a GPA $\overline{\mathcal{G}}$ iff for any concrete states $s, s' \in S$ and any concrete action $a \in A$ whenever $\pi(s) = a$ and $t(s, a, s') > 0$ there exists a hyperedge $\overline{e} \equiv \langle e_\emph{src}, e_\emph{dest}, e_\emph{act} \rangle \in \overline{\mathcal{E}}$ where $\overline{s}  = \overline{e}_\emph{src}$, $\overline{s}' \in \overline{e}_\emph{dest}$, and $\overline{a} = \overline{e}_\emph{act}$. 
\end{definition}

\subsubsection{Learning GPAs}
\label{subsec:learning_gpga}

It is well-known that solutions to small problems can be used to construct generalized control structures that can assist in solving larger problems. We adopt a similar strategy of the learn-from-small-examples approach \citep{DBLP:conf/aips/WuG07,DBLP:conf/aaai/KariaS21} and 
compute GPAs iteratively from a small training set containing solutions of similar SSP instances as outlined below.

To form our training set, we use a library of solution policies $\Pi = \{ \pi_1, \ldots, \pi_n\}$ for \emph{small} SSPs $P_1, \ldots, P_n$ that can be easily (and optimally) computed by existing state-of-art SSP solvers. We use the transition function for $P_i$ to convert each policy $\pi_i \in \Pi$ to a set of transitions $\tau_i = \{ (s, a, s') | s, s' \in S_i, a \in A_i, \pi(s) = a, t(s, a, s') > 0\}$. We then construct our training set $\mathcal{T} = \{ \tau_1, \ldots, \tau_n \}$ as a set of concrete transitions.

The GPA $\overline{\mathcal{G}}_\Pi$ learned from a finite set of concrete transitions $\mathcal{T}$ is defined as follows.
We first initialize a GPA $\overline{\mathcal{G}}_\Pi = \langle \{ \overline{\mathcal{V}}, \overline{\mathcal{E}} \} \rangle = \langle \{\}, \{\} \rangle$.
Next, we convert $\mathcal{T}$ into an abstract transition set $\overline{\mathcal{T}} = \{ (\alpha(F_\alpha, s), \beta(F_\beta, s, a), \alpha(F_\alpha, s')) | (s, a, s') \in \mathcal{T} \}$. We then form the vertex set by using all abstract states in $\overline{\mathcal{T}}$, i.e., $\overline{\mathcal{V}} = \overline{\mathcal{V}} \cup \{ \overline{s}, \overline{s}'\}$ for every $(\overline{s}, \overline{a}, \overline{s}') \in \overline{\mathcal{T}}$. Similarly, we convert each abstract transition into a hyperedge and add it to the edge set, i.e. $\overline{\mathcal{E}} = \overline{\mathcal{E}} \cup \langle \overline{s}, \{ \overline{s}' \}, \overline{a} \rangle$ for every $(\overline{s}, \overline{a}, \overline{s}') \in \overline{\mathcal{T}}$. Once this is done, we compress $\overline{\mathcal{G}}_\Pi$ by replacing any edges in $\overline{\mathcal{E}}$ that have the same start nodes and labels but different destinations with a single new edge that combines the destinations of the edges, i.e., for any two edges $\overline{e}^1, \overline{e}^2 \in \overline{\mathcal{E}}$ s.t. $\overline{e}^1_\emph{src} = \overline{e}^2_\emph{src}$, $\overline{e}^1_\emph{dest} \not= \overline{e}^2_\emph{dest}$, and $\overline{e}^1_\emph{act} = \overline{e}^2_\emph{act}$, $\overline{\mathcal{E}} = \overline{\mathcal{E}} \setminus \{ \overline{e}^1, \overline{e}^2 \} \cup \langle \overline{e}^1_\emph{src}, \overline{e}^1_\emph{dest} \cup \overline{e}^2_\emph{dest}, \overline{e}^1_\emph{act} \rangle$. Fig.\,\ref{fig:merge_gpas} provides a high-level overview of our procedure for merging GPAs. Henceforth, we drop the subscript $\Pi$ from $\overline{\mathcal{G}}_\Pi$ when it is clear from context.


\begin{algorithm}[t]

\begin{algorithmic}[1]
\REQUIRE SSP $P = \langle O, S, A, s_0, g, T, C \rangle$, GPA $\overline{\mathcal{G}} = \langle \overline{\mathcal{V}}, \overline{\mathcal{E}} \rangle$,  \\ Feature Sets $F_\alpha, F_\beta$,  Abstraction Functions $\alpha, \beta$

\STATE $C' = C$ \COMMENT{copy over the cost function of $P$}
\FOR {$(s, a, s') \in S \times A \times S$}
    \STATE $\overline{s}, \overline{a}, \overline{s'} \gets \alpha(F_\alpha, s), \beta(F_\beta, s, a), \alpha(F_\alpha, s')$
    \IF {there is no edge $\overline{e} \in \overline{\mathcal{E}}$ s.t. $\overline{e}_\emph{src} = \overline{s}, \overline{s'} \in \overline{e}_\emph{dest}$, and  $\overline{e}_\emph{act} = \overline{a}$}
        \STATE $C'[s, a, s'] = \infty$
    \ENDIF
\ENDFOR

\STATE $P|_{\overline{\mathcal{G}}} = \langle O, S, A, s_0, g, T, C' \rangle$
\STATE $V^{*}_{P|_{\overline{\mathcal{G}}}} \gets \text{initializeValueFunction}()$
\STATE $V^{*}_{P|_{\overline{\mathcal{G}}}}, \pi^{*}_{P|_{\overline{\mathcal{G}}}} \gets \text{optimallySolveSSP}(P|_{\overline{\mathcal{G}}}, V^{*}_{P|_{\overline{\mathcal{G}}}})$

\IF{$\pi^{*}_{P|_{\overline{\mathcal{G}}}}$ is a partial proper policy}
    \RETURN $\pi^{*}_{P|_{\overline{\mathcal{G}}}}$
\ELSE
    \STATE $V_{P}, \pi_{P} \gets \text{optimallySolveSSP}(P, V^{*}_{P|_{\overline{\mathcal{G}}}})$
    \RETURN $\pi_{P}$
\ENDIF
\end{algorithmic}
\caption{GPA acceleration for SSPs}
\label{alg:convert_ssp}
\end{algorithm}

\subsection{Solving SSPs using GPAs}
\label{subsec:solving_ssp}

The key intuition behind our method is to use the GPA to prune transitions that are not consistent with the GPA from the search process. We accomplish this by solving a GPA constrained problem that encodes the constrained encoded by the GPA. We define a GPA constrained problem as follows:
\begin{definition}[GPA constrained problem]

Let $P = \langle O, S, A, T, C, s_0, g \rangle$ be an SSP for a domain $D$ and let $\overline{\mathcal{G}} = \langle \overline{\mathcal{V}}, \overline{\mathcal{E}} \rangle$ be a GPA. A GPA constrained problem $P|_{\overline{\mathcal{G}}} = \langle O, S, A, T, C', s_0, g \rangle$ is defined with a cost function $C': S \times A \times S \rightarrow \mathbb{R}^+$ such that $C'[s, a, s'] = C[s, a, s']$ when there exists a hyperedge $\overline{e} \equiv \langle e_\emph{src}, e_\emph{dest}, e_\emph{act} \rangle \in \overline{\mathcal{E}}$ where $\overline{s}  = \overline{e}_\emph{src}$, $\overline{s}' \in \overline{e}_\emph{dest}$, and $\overline{a} = \overline{e}_\emph{act}$ and $C'[s, a, s'] = \infty$ if there is no such hyperedge.
\end{definition}

Given a GPA $\overline{\mathcal{G}}$, Alg.\,\ref{alg:convert_ssp} works as follows. Lines 1-8 create a GPA constrained problem $P|_{\overline{\mathcal{G}}}$. Next, line 10 optimally solves $P|_{\overline{\mathcal{G}}}$ using any off-the-shelf SSP solver using a randomly initialized value function (line 9). If the computed policy $\pi^{*}_{P|_{\overline{\mathcal{G}}}}$ is a partial proper policy then it is returned immediately (lines 11-12) else Alg.\,\ref{alg:convert_ssp} use a new instance of the SSP solver to compute a policy for $P$ using $V^{*}_{P|_{\overline{\mathcal{G}}}}$ as the \emph{bootstrapped} initial value estimates. Information such as whether a state is a dead end etc., is not carried over. Line 15 then optimally solves $P$ using $V^{*}_{P|_{\overline{\mathcal{G}}}}$ as the initial value estimates and returns a partial proper policy $\pi_P$ for $P$. Since $V^{*}_{P|_{\overline{\mathcal{G}}}}$ is only used as an initial bootstrapping estimate for $P$, an SSP solver will only return a $\pi_P$ is better than or equal to $\pi^{*}_{P|_{\overline{\mathcal{G}}}}$ following standard results on policy improvement for value iteration \cite{DBLP:books/lib/SuttonB98}. 

The goal of modifying the cost function is to prevent concrete transitions whose abstract translations are not present in the GPA to be used when performing Bellman updates for $P|_{\overline{\mathcal{G}}}$. As a result, actions belonging to such transitions cannot appear in $\pi^{*}_{P|_{\overline{\mathcal{G}}}}$ since their costs would be $\infty$. As a consequence of doing so, the existence of a partial proper policy is not guaranteed in $P|_{\overline{\mathcal{G}}}$. 

Every optimal policy $\pi^{*}_{P|_{\overline{\mathcal{G}}}}$ for $P|_{\overline{\mathcal{G}}}$ is said to be \emph{hierarchically optimal} for $P$ given $\overline{\mathcal{G}}$. Alg.\,\ref{alg:convert_ssp} computes an optimal policy $\pi^{*}_{P|_{\overline{\mathcal{G}}}}$ for $P|_{\overline{\mathcal{G}}}$ in the space of cross-product of the states of the GPA $\overline{\mathcal{G}}$ with the states of the SSP $P$, similar to HAMs \citep{DBLP:conf/nips/ParrR97}. As seen in our empirical analysis in Sec.\,\ref{sec:experiments}, we observe that using a small set of example policies that are sufficient to capture rich generalized control structures that are encoded by such hierarchically optimal policies. Alg.\,\ref{alg:convert_ssp} computes such policies within a fraction of the original computational effort and in most cases with costs comparable to $\pi^*_P$. Alg.\,\ref{alg:convert_ssp} is complete in that it will always return a partial proper policy for an SSP $P$. We now state and prove some theoretical guarantees of completeness and hierarchical optimality below.\footnote{Complete proofs are available in the extended version \citep{arXiv:kns_neurips22}.}

The following result indicates that Alg.\,\ref{alg:convert_ssp} is a complete algorithm for solving SSPs.
\begin{theorem}
\label{theorem:completeness}
Given a GPA $\overline{\mathcal{G}}$ and an SSP $P$, Alg.\,\ref{alg:convert_ssp} always returns a partial proper policy.
\end{theorem}

\begin{proof}[Proof]
Alg.\,\ref{alg:convert_ssp} returns $\pi^{*}_{P|_{\overline{\mathcal{G}}}}$ iff it is a partial proper policy (line 11) else $V^{*}_{P|_{\overline{\mathcal{G}}}}(s_0) = \infty$ and Alg.\,\ref{alg:convert_ssp} returns a policy for $\pi_P$ by solving $P$ using initial value estimates from $V^{*}_{P|_{\overline{\mathcal{G}}}}$. Since the existence of a partial proper policy in $P$ is guaranteed by definition, Alg.\,\ref{alg:convert_ssp} will find it (line 14) and return it.
\end{proof}

The following result indicates that the output of Alg.\,\ref{alg:convert_ssp} is hierarchically optimal or better.

\begin{theorem}
\label{theorem:quality_of_policy}
Let $V^{\pi}$ be the value function for a policy $\pi$ returned by Alg.\,\ref{alg:convert_ssp} for an SSP $P$ using GPA $\overline{\mathcal{G}}$. Let $V^{*}_{P|_{\overline{\mathcal{G}}}}$ be the optimal value function for $P|_{\overline{\mathcal{G}}}$, then $V^{\pi}(s_0) \le V^{*}_{P|_{\overline{\mathcal{G}}}}(s_0)$.
\end{theorem}

\begin{proof}[Proof (Sketch)]
If Alg.\,\ref{alg:convert_ssp} finds a partial proper policy  $\pi^{*}_{P|_{\overline{\mathcal{G}}}}$ for $P|_{\overline{\mathcal{G}}}$ then it returns it immediately (line 11) in which case $V^{\pi}(s_0) = V^{*}_{P|_{\overline{\mathcal{G}}}}(s_0)$. If $\pi^{*}_{P|_{\overline{\mathcal{G}}}}$ is not a partial proper policy then $V^{*}_{P|_{\overline{\mathcal{G}}}}(s_0) = \infty$. Since Alg.\,\ref{alg:convert_ssp} is complete (Thm.\,\ref{theorem:completeness}), $\pi$ is a partial proper policy where $V^{\pi}(s_0) < \infty$.
\end{proof}

\begin{corollary}
\label{theorem:optimality}
If $V^{*}_{P}(s_0) = V^{*}_{P|_{\overline{\mathcal{G}}}}(s_0)$, then Alg.\,\ref{alg:convert_ssp} returns the optimal policy for $P$.
\end{corollary}

The following result indicates that only a finite set of training examples are needed to learn a GPA such that the GPA constrained problem will always yield the optimal policy for a given domain $D$ and goal formula $g$.

\begin{theorem}
\label{theorem:optimality_in_the_limit}
Suppose $D$ is a domain, $g$ is a goal formula over the predicates in $D$'s vocabulary, $\overline{\mathcal{G}}_{\Pi^{*}}$ is a GPA such that for every SSP instance $P$ of $D$ and $g$, there exists an optimal policy $\pi^*_P$ that is consistent with $\overline{\mathcal{G}}_{\Pi^{*}}$. Then there exists a finite set of training policies $\Pi^{*}$ from which $\overline{\mathcal{G}}_{\Pi^{*}}$ can be learned.

\end{theorem}

\begin{proof}[Proof (Sketch)]

Since the sizes of the abstract state and action spaces $\overline{S}$ and $\overline{A}$ are finite, the size of $\overline{\mathcal{G}}_{\Pi^{*}}$ is finite and is bounded by $|\overline{S}| \times |\overline{A}| \times |\overline{S}|$. For every $\overline{e} \in \overline{\mathcal{E}}$ at most $|e_\emph{dest}|$ different transitions are needed to learn $\overline{e}$. Since there is a finite number of edges a finite amount of training data will suffice for learning $\overline{\mathcal{G}}_{\Pi^{*}}$.
\end{proof}

In the worst case the cost functions of $P$ and $P|_{\overline{\mathcal{G}}}$ are similar and no savings can be obtained. However, in our empirical evaluation, we observed that typically very few and small training problems suffice to learn a compact GPA $\overline{\mathcal{G}}$ that allows cheap computation of good quality policies for problems that are significantly larger than those used during training. Finding the right set of examples $\Pi^{*}$ from which $\overline{\mathcal{G}}_{\Pi^{*}}$ is an interesting and non-trivial research problem that we leave to future work.

\section{Experiments}
\label{sec:experiments}

We conducted an empirical evaluation on five well-known benchmark domains that were selected from the International Planning Competition (IPC) and International Probabilistic Planning Competition (IPPC) \citep{DBLP:journals/jair/YounesLWA05} as well as robotic planning \citep{shah2020anytime}. As a part of our analysis, we use the time required to compute a solution and measure the quality of the found solutions to determine if GPAs allow efficient solving of SSPs.

We chose PPDDL as our representational language, which was the default language in IPPCs until 2011, after which, the Relational Dynamic Influence Diagram Language (RDDL) \citep{Sanner:RDDL} became the default. We chose PPDDL over RDDL since RDDL does not allow specifying the goal condition easily and as a result many benchmarks using RDDL are general purpose MDPs with no goals, and since modern state-of-art solvers for PPDDL are available.

For our baselines, we focus on complete solvers for SSPs. Deep learning based approaches do not guarantee completeness and thus are not directly comparable with our work. We used Labeled RTDP (LRTDP) \citep{DBLP:conf/aips/BonetG03} and Soft-FLARES \citep{DBLP:conf/atal/PinedaZ19} which are state-of-art (SOA), complete SSP solvers. These algorithms internally generate their own heuristics using the domain and problem file. We used the inadmissible FF heuristic \citep{DBLP:journals/aim/Hoffmann01} as the internal heuristic function for all algorithms since the baselines performed best using it.

We ran our experiments on a cluster of Intel Xeon E5-2680 v4 CPUs running at 2.4 GHz with 16 GiB of RAM. 
Our implementation is in Python and we ported C++ implementations of the baselines from \citet{DBLP:conf/atal/PinedaZ19} to Python.\footnote{Our source code is available at \url{https://github.com/AAIR-lab/GRAPL}} We utilized problem generators from the IPPC suite and those in \citet{shah2020anytime} for generating the training and test problems for all domains. We provide a brief description of the problem domains below.
\\
\textbf{Rover$(r, w, s, o)$} A set of $r$ rovers need to collect and drop $s$ samples that are present at one of $w$ waypoints. The rovers also need to collect images of different objectives $o$ that are visible from certain waypoints. This is an IPC domain and we converted it into a stochastic version by modifying sample collecting actions to fail with a probability of $0.4$ (keeping the rover in the same state).
\\
\textbf{Gripper$(b)$} A robot with two grippers is placed in an environment consisting of two rooms A and B. The objective of the robot is to transfer all the balls $b$ initially located in room A to room B. We modified the gripper to be slippery so that picking balls have a 20\% chance of failure.
\\
\textbf{Schedule$(C, p)$} is an IPPC domain that consists of a set of $p$ packets each belonging to one of $C$ different classes that need to be queued. A router must first process the arrival of a packet in order to route it. The interval at which the router processes arrivals is determined by probability 0.94.
\\
\textbf{Keva$(P, h)$} A robot uses $P$ keva planks to build a tower of height $h$. Planks are placed in a specific order in one of the two  locations preferring one location with probability $0.6$.  Despite this simple setting, the keva domain has been shown to be a challenging robotics problem \citep{shah2020anytime}.
\\
\textbf{Delicate Can$(c)$} An arrangement of $c$ cans on a table of which 1 is a delicate can. The objective is to pick up a specific goal can. Cans can obstruct the trajectory to the goal can and they must be moved in order to successfully pick up the goal can. Cans can be crushed with probability 0.1 (delicate cans have a higher chance with probability 0.8) and need to be revived.

\textbf{Training Setup} Our method learns GPAs in a few-shot fashion, requiring little to no training data. For our training set, we used at most ten optimal solution policies (obtained using LAO*) for each domain.  The time required to learn a GPA was less than 10 seconds in all cases in our experiments. This highlights the advantages of GPAs that can quickly be learned in a few-shot setting. Moreover, compared to neuro-symbolic methods, GPAs are not subject to catastrophic forgetting and new training data can easily be merged with the existing GPA using the process described in Sec.\,\ref{subsec:learning_gpga}.\\
\textbf{Test Setup} We fixed the time and memory limit for each problem to $7200$ seconds and $16$GiB respectively. To demonstrate generalizability, our test set contains problems with object counts that are much larger than the training policies used. The largest problems in our test sets contain at least twice the number of objects than those used in training. For example, in the Keva domain we use training policies with towers of height up to 6 and evaluate on problems with towers of height up to 14. The minimum and maximum number of problems that we used in our test set are 11 and 26 problems respectively.  Due to space limitations, information pertaining to the total number of training and test problems and their parameters, used hyperparameters for configuring baselines etc., are included in the appendix.

\begin{figure}[t]
\centering
\includegraphics[width=\linewidth]{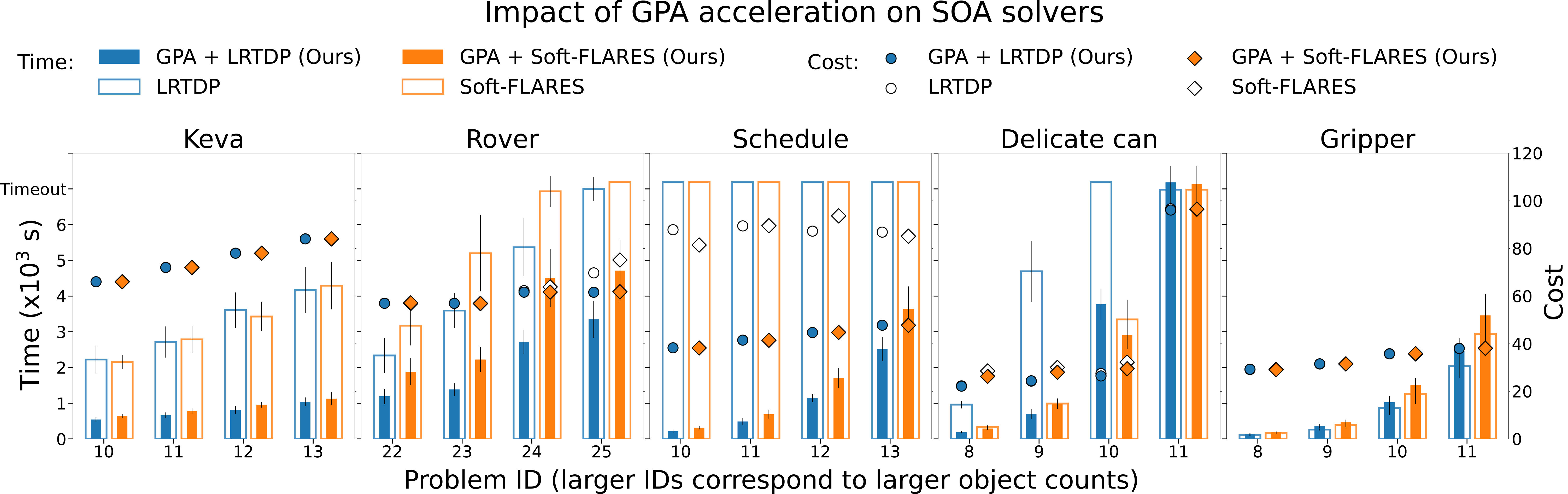}
\caption{Impact of learned GPAs on solver performance (lower values are better). Left y-axes and bars show solution times (in units of 1000 secs) for our approach and baseline SOA solvers (LRTDP and Soft-FLARES). Right y-axes and points show cost incurred by the policy computed by our approach and the baselines.
 We use the same SSP solver as the corresponding baseline in our approach. Error bars indicate 1 standard deviation (SD) averaged across 10 runs. For clarity, we only report results for the largest test problems and omit error bars from costs due to the low SDs. A complete view of our empirical results is available in the appendix.}
    \label{fig:results}
\end{figure}

\subsection{Results and Analysis}

Our evaluation metric compares the time required to find a partial proper policy. We also compare the quality of the computed policies by executing the policies for 100 trials with a horizon limit of 100 and averaging the obtained costs. We report our overall results averaged across 10 different runs and report results up to 1 standard deviation. Results of our experiments are illustrated in Fig.\,\ref{fig:results}. 

In four out of five domains (Schedule, Rover, Keva, and Delicate Can), our approach takes significantly less time compared to the corresponding baseline. For example, in Schedule, the baselines timed out for all of the large test problems reported. GPAs are able to successfully prune away action transitions that are not helping, leading to large savings in the computational effort. The costs obtained for executing these policies are also quite similar to each baseline showing that GPAs are capable of learning good policies much faster without compromising solution quality.

Our approach was unable to outperform the baselines in the Gripper domain. An interesting phenomenon that we observed was that training policies returned by LAO* were different for the case of even/odd balls due to tie-breaking and this led to the GPA not pruning actions as effectively. Nevertheless, we expected GPA to still outperform the baselines. We performed a deeper investigation and found that the FF heuristic used is already well-suited to prune away actions that the GPA would have otherwise pruned. This results in additional overhead being added in our SSP solver from the process of abstraction. However, heuristics that allow such pruning are difficult to synthesize, and in many cases, are hand-coded by an expert after employing significant effort.

Finally, because of the fixed timeout used, the maximum time of the baselines was bounded, making the impact of GPA appear smaller in larger problems. For example, in problem ID 25 of the Rover domain, the Soft-FLARES baseline timed out in all our runs, but 
when allowed to run to convergence, it took over 15000 seconds in a targeted experiment that we performed for investigating this issue.

\section{Related Work}
\label{sec:related_work}

There has been plenty of dedicated research to improve the efficiency for solving SSPs. LAO$^*$ \citep{hansen2001lao} computes policies by using heuristics to guide the search process. LRTDP \citep{DBLP:conf/aips/BonetG03} uses a labeling procedure in RTDP wherein a part of the subtree that is $\epsilon$-consistent is marked as \emph{solved} leading to faster ending of trials.  SSiPP \citep{DBLP:conf/aips/TrevizanV12} uses short-sightedness by only considering reachable states up to $t$ states away and solving this constrained SSP. Soft-FLARES \citep{DBLP:conf/atal/PinedaZ19} combines labeling and short-sightedness for computing solutions. These approaches are complete and can be configured to return optimal solutions, however, they fail to learn any generalized knowledge and as result cannot readily scale to larger problems with a greater number of objects.

\citet{DBLP:conf/aips/PinedaZ14} build sparse representations of SSP problems by reducing the branching factor in the ``environment’s choices'' (the set of probabilistic effects of an action), while our approach uses abstraction to create abstract controllers that generalize solutions to SSPs and reduces the branching factor in the agent’s choice (the set of applicable actions). Our approach always considers all possible outcomes of every action. This is a key advantage of our approach since GPAs are able to better handle unexpected outcomes when executing an action in the policy.

\citet{DBLP:conf/ijcai/BoutilierRP01} utilize decision-theoretic regression to compute generalized policies for first-order MDPs represented using situation calculus. They utilize symbolic dynamic programming to compute a symbolic value function that applies to problems with varying number of objects. FOALP \citep{DBLP:conf/uai/SannerB05} uses linear programming to compute an approximation of the value function for first-order MDPs while providing upper bounds on the approximation error irrespective of the domain size. A key limitation of their approach is requiring the use of a representation of action models over which it is possible to regress using situation calculus. API \citep{DBLP:journals/jair/FernYG06} uses approximate policy iteration with taxonomic decision lists to form policies. They use Monte Carlo simulations with random walks on a single problem to construct a policy. API offers no guarantees of completeness or hierarchical optimality.

\citet{DBLP:conf/nips/ParrR97} propose the hierarchical abstract machine (HAM) framework wherein component solutions from problem instances can be combined to solve larger problem instances efficiently. Recently, \citet{DBLP:conf/ijcai/BaiR17} extended HAMs to Reinforcement Learning  settings by leveraging internal transitions of the HAMs. A key limitation of both these approaches is that the HAMs were hand-coded by a domain expert. 

Related work in Generalized Planning focuses on the problem of computing generalized plans and policies such as our GPAs \cite{DBLP:journals/ai/SrivastavaIZ11}.
\citet{DBLP:conf/aips/BonetPG09} automatically create finite-state controllers for solving problems using a set of examples by modeling the search as a contingent problem. Their approach is limited in applicability since it only works on deterministic problems and the features they use are hand-coded. \citet{DBLP:conf/ijcai/AguasCJ16} utilize small example policies to synthesize hierarchical finite state controllers that can call each other. However, their approach requires all training data to be provided upfront. D2L \citep{DBLP:conf/aaai/FrancesBG21} utilizes description logics to automatically generate features and reactive policies based on those features. Their approach comes with no guarantees for finding a solution or its cost and can only work on deterministic problems.

Deep learning based approaches such as ASNets \cite{DBLP:conf/aaai/ToyerTTX18} learn generalized policies for SSPs using a neuro-symbolic approach. They use the action schema from PPDDL to create alternating action and proposition layers. They do not learn learn generalized controllers and instead duplicate weights in a post-process step to represent the generalized policy. A key limitation of their approach is the lack of any theoretical guarantees of completeness or optimality.

Our approach differs from these approaches in several aspects. Our approach constructs a GPA automatically without any human intervention. Using canonical abstraction, we lift problem-specific characteristics like object names and object counts. Another key difference between other techniques is that our approach can easily incorporate solutions from new examples into the GPA without having to remember any of the earlier examples. This allows our learning to scale better and can naturally utilize \emph{leapfrog learning} \citep{DBLP:conf/aips/GroshevGTSA18,DBLP:conf/aaai/KariaS21} when presented with a large problem in the absence of training data. Finally, our approach comes with guarantees of completeness and hierarchical optimality given the training data presented. This means that if a solution exists, our approach will find it and it will be guaranteed to be hierarchically optimal.

\section{Conclusions and Future Work}
\label{sec:conclusions}

We show that non-deterministic Generalized Policy Automata (GPAs) constructed using solutions of small example SSPs are able to significantly reduce the computational effort for finding solutions for larger related SSPs. Furthermore, for many benchmark problems, the search space pruned by GPAs does not prune away relevant transitions allowing our approach to compute policies of comparable cost in a fraction of the effort. Our approach comes with guarantees of hierarchical optimality and also comes with the guarantee of always finding a solution to the SSP.

There are several interesting research directions for future work. Description Logics (DL) are more expressive than canonical abstractions and have been demonstrated to be effective at synthesizing memoryless controllers for deterministic planning problems \citet{DBLP:conf/aaai/BonetFG19}. Our approach can easily utilize any relational abstraction and it would be interesting to evaluate the efficacy of description logics. Finally, our approach is applicable when solutions have a pattern. We believe that more intelligent training data generation methods could help improve performance in domains like Delicate Can. We plan to investigate these directions of research in future work.

\section*{Acknowledgements}
We would like to thank Deepak Kala Vasudevan for help with a prototype implementation of the source code. We would like to thank the Research Computing Group at Arizona State University for providing compute hours for our experiments. This work was supported in part by the NSF under grants IIS 1909370 and IIS 1942856.

\bibliographystyle{abbrvnat}
\bibliography{neurips_2022}

\begin{thebibliography}{34}
\providecommand{\natexlab}[1]{#1}
\providecommand{\url}[1]{\texttt{#1}}
\expandafter\ifx\csname urlstyle\endcsname\relax
  \providecommand{\doi}[1]{doi: #1}\else
  \providecommand{\doi}{doi: \begingroup \urlstyle{rm}\Url}\fi

\bibitem[Aguas et~al.(2016)Aguas, Celorrio, and
  Jonsson]{DBLP:conf/ijcai/AguasCJ16}
J.~S. Aguas, S.~J. Celorrio, and A.~Jonsson.
\newblock Hierarchical finite state controllers for generalized planning.
\newblock In \emph{Proc. {IJCAI}}, 2016.
\newblock URL \url{http://www.ijcai.org/Abstract/16/458}.

\bibitem[Bai and Russell(2017)]{DBLP:conf/ijcai/BaiR17}
A.~Bai and S.~J. Russell.
\newblock Efficient reinforcement learning with hierarchies of machines by
  leveraging internal transitions.
\newblock In \emph{{IJCAI}}, 2017.

\bibitem[Bertsekas and Tsitsiklis(1996)]{DBLP:books/lib/BertsekasT96}
D.~P. Bertsekas and J.~N. Tsitsiklis.
\newblock \emph{Neuro-dynamic programming}.
\newblock Athena Scientific, 1996.
\newblock ISBN 1886529108.

\bibitem[Bonet and Geffner(2003)]{DBLP:conf/aips/BonetG03}
B.~Bonet and H.~Geffner.
\newblock Labeled {RTDP:} improving the convergence of real-time dynamic
  programming.
\newblock In \emph{{ICAPS}}, 2003.

\bibitem[Bonet et~al.(2009)Bonet, Palacios, and
  Geffner]{DBLP:conf/aips/BonetPG09}
B.~Bonet, H.~Palacios, and H.~Geffner.
\newblock Automatic derivation of memoryless policies and finite-state
  controllers using classical planners.
\newblock In \emph{{ICAPS}}, 2009.

\bibitem[Bonet et~al.(2019)Bonet, Franc{\`{e}}s, and
  Geffner]{DBLP:conf/aaai/BonetFG19}
B.~Bonet, G.~Franc{\`{e}}s, and H.~Geffner.
\newblock Learning features and abstract actions for computing generalized
  plans.
\newblock In \emph{{AAAI}}, 2019.

\bibitem[Boutilier et~al.(2001)Boutilier, Reiter, and
  Price]{DBLP:conf/ijcai/BoutilierRP01}
C.~Boutilier, R.~Reiter, and B.~Price.
\newblock Symbolic dynamic programming for first-order mdps.
\newblock In \emph{{IJCAI}}, 2001.

\bibitem[Fern et~al.(2006)Fern, Yoon, and Givan]{DBLP:journals/jair/FernYG06}
A.~Fern, S.~W. Yoon, and R.~Givan.
\newblock Approximate policy iteration with a policy language bias: Solving
  relational markov decision processes.
\newblock \emph{J. Artif. Intell. Res.}, 25:\penalty0 75--118, 2006.

\bibitem[Franc{\`{e}}s et~al.(2021)Franc{\`{e}}s, Bonet, and
  Geffner]{DBLP:conf/aaai/FrancesBG21}
G.~Franc{\`{e}}s, B.~Bonet, and H.~Geffner.
\newblock Learning general planning policies from small examples without
  supervision.
\newblock In \emph{{AAAI}}, 2021.

\bibitem[Garg et~al.(2020)Garg, Bajpai, and Mausam]{DBLP:conf/icml/GargBM20}
S.~Garg, A.~Bajpai, and Mausam.
\newblock Symbolic network: Generalized neural policies for relational mdps.
\newblock In \emph{{ICML}}, 2020.

\bibitem[Groshev et~al.(2018{\natexlab{a}})Groshev, Goldstein, Tamar,
  Srivastava, and Abbeel]{DBLP:conf/aips/GroshevGTSA18}
E.~Groshev, M.~Goldstein, A.~Tamar, S.~Srivastava, and P.~Abbeel.
\newblock Learning generalized reactive policies using deep neural networks.
\newblock In \emph{{ICAPS}}, 2018{\natexlab{a}}.

\bibitem[Groshev et~al.(2018{\natexlab{b}})Groshev, Tamar, Goldstein,
  Srivastava, and Abbeel]{groshev2018learning}
E.~Groshev, A.~Tamar, M.~Goldstein, S.~Srivastava, and P.~Abbeel.
\newblock Learning generalized reactive policies using deep neural networks.
\newblock In \emph{AAAI Spring Symposium Series}, 2018{\natexlab{b}}.

\bibitem[Hansen and Zilberstein(2001)]{hansen2001lao}
E.~A. Hansen and S.~Zilberstein.
\newblock Lao\*: A heuristic search algorithm that finds solutions with loops.
\newblock \emph{Artificial Intelligence}, 129\penalty0 (1-2):\penalty0 35--62,
  2001.

\bibitem[Hoffmann(2001)]{DBLP:journals/aim/Hoffmann01}
J.~Hoffmann.
\newblock {FF:} the fast-forward planning system.
\newblock \emph{{AI} Mag.}, 22\penalty0 (3):\penalty0 57--62, 2001.

\bibitem[Karia and Srivastava(2021)]{DBLP:conf/aaai/KariaS21}
R.~Karia and S.~Srivastava.
\newblock Learning generalized relational heuristic networks for model-agnostic
  planning.
\newblock In \emph{{AAAI}}, 2021.

\bibitem[Karia et~al.(2022)Karia, Nayyar, and Srivastava]{arXiv:kns_neurips22}
R.~Karia, R.~K. Nayyar, and S.~Srivastava.
\newblock Learning generalized policy automata for relational stochastic
  shortest path problems.
\newblock \emph{arXiv}, abs/2204.04301, 2022.

\bibitem[Kolobov et~al.(2012)Kolobov, Mausam, and
  Weld]{DBLP:conf/uai/KolobovMW12}
A.~Kolobov, Mausam, and D.~S. Weld.
\newblock A theory of goal-oriented mdps with dead ends.
\newblock In \emph{UAI}, 2012.

\bibitem[Levy et~al.(2019)Levy, Konidaris, Platt, and Saenko]{levy2017learning}
A.~Levy, G.~Konidaris, R.~Platt, and K.~Saenko.
\newblock Learning multi-level hierarchies with hindsight.
\newblock In \emph{ICLR}, 2019.

\bibitem[Parr and Russell(1997)]{DBLP:conf/nips/ParrR97}
R.~Parr and S.~J. Russell.
\newblock Reinforcement learning with hierarchies of machines.
\newblock In \emph{{NeurIPS}}, 1997.

\bibitem[Pineda and Zilberstein(2014)]{DBLP:conf/aips/PinedaZ14}
L.~E. Pineda and S.~Zilberstein.
\newblock Planning under uncertainty using reduced models: Revisiting
  determinization.
\newblock In \emph{{ICAPS}}, 2014.
\newblock URL
  \url{http://www.aaai.org/ocs/index.php/ICAPS/ICAPS14/paper/view/7920}.

\bibitem[Pineda and Zilberstein(2019)]{DBLP:conf/atal/PinedaZ19}
L.~E. Pineda and S.~Zilberstein.
\newblock Soft labeling in stochastic shortest path problems.
\newblock In \emph{{AAMAS}}, 2019.

\bibitem[Pong et~al.(2018)Pong, Gu, Dalal, and Levine]{pong2018temporal}
V.~Pong, S.~Gu, M.~Dalal, and S.~Levine.
\newblock Temporal difference models: Model-free deep rl for model-based
  control.
\newblock In \emph{ICLR}, 2018.

\bibitem[Sagiv et~al.(2002)Sagiv, Reps, and
  Wilhelm]{DBLP:journals/toplas/SagivRW02}
S.~Sagiv, T.~W. Reps, and R.~Wilhelm.
\newblock Parametric shape analysis via 3-valued logic.
\newblock \emph{{ACM} Trans. Program. Lang. Syst.}, 24\penalty0 (3):\penalty0
  217--298, 2002.

\bibitem[Sanner(2010)]{Sanner:RDDL}
S.~Sanner.
\newblock Relational dynamic influence diagram language ({RDDL}): Language
  description.
\newblock \url{http://users.cecs.anu.edu.au/~ssanner/IPPC_2011/RDDL.pdf}, 2010.

\bibitem[Sanner and Boutilier(2005)]{DBLP:conf/uai/SannerB05}
S.~Sanner and C.~Boutilier.
\newblock Approximate linear programming for first-order mdps.
\newblock In \emph{{UAI}}, 2005.

\bibitem[Shah et~al.(2020)Shah, Vasudevan, Kumar, Kamojjhala, and
  Srivastava]{shah2020anytime}
N.~Shah, D.~K. Vasudevan, K.~Kumar, P.~Kamojjhala, and S.~Srivastava.
\newblock Anytime integrated task and motion policies for stochastic
  environments.
\newblock In \emph{{ICRA}}, 2020.

\bibitem[Srivastava et~al.(2008)Srivastava, Immerman, and
  Zilberstein]{DBLP:conf/aaai/SrivastavaIZ08}
S.~Srivastava, N.~Immerman, and S.~Zilberstein.
\newblock Learning generalized plans using abstract counting.
\newblock In \emph{Proc. {AAAI}}, 2008.
\newblock URL \url{http://www.aaai.org/Library/AAAI/2008/aaai08-157.php}.

\bibitem[Srivastava et~al.(2011)Srivastava, Immerman, and
  Zilberstein]{DBLP:journals/ai/SrivastavaIZ11}
S.~Srivastava, N.~Immerman, and S.~Zilberstein.
\newblock A new representation and associated algorithms for generalized
  planning.
\newblock \emph{Artif. Intell.}, 175\penalty0 (2):\penalty0 615--647, 2011.

\bibitem[Sutton and Barto(1998)]{DBLP:books/lib/SuttonB98}
R.~S. Sutton and A.~G. Barto.
\newblock \emph{Reinforcement learning - an introduction}.
\newblock {MIT} Press, 1998.
\newblock ISBN 978-0-262-19398-6.

\bibitem[Tamar et~al.(2016)Tamar, Wu, Thomas, Levine, and
  Abbeel]{tamar2016value}
A.~Tamar, Y.~Wu, G.~Thomas, S.~Levine, and P.~Abbeel.
\newblock Value iteration networks.
\newblock \emph{In NeurIPS}, 29, 2016.

\bibitem[Toyer et~al.(2018)Toyer, Trevizan, Thi{\'{e}}baux, and
  Xie]{DBLP:conf/aaai/ToyerTTX18}
S.~Toyer, F.~W. Trevizan, S.~Thi{\'{e}}baux, and L.~Xie.
\newblock Action schema networks: Generalised policies with deep learning.
\newblock In \emph{{AAAI}}, 2018.

\bibitem[Trevizan and Veloso(2012)]{DBLP:conf/aips/TrevizanV12}
F.~W. Trevizan and M.~M. Veloso.
\newblock Short-sighted stochastic shortest path problems.
\newblock In \emph{{ICAPS}}, 2012.

\bibitem[Wu and Givan(2007)]{DBLP:conf/aips/WuG07}
J.~Wu and R.~Givan.
\newblock Discovering relational domain features for probabilistic planning.
\newblock In \emph{{ICAPS}}, 2007.

\bibitem[Younes et~al.(2005)Younes, Littman, Weissman, and
  Asmuth]{DBLP:journals/jair/YounesLWA05}
H.~L.~S. Younes, M.~L. Littman, D.~Weissman, and J.~Asmuth.
\newblock The first probabilistic track of the international planning
  competition.
\newblock \emph{J. Artif. Intell. Res.}, 24:\penalty0 851--887, 2005.

\end{thebibliography}

\appendix

\clearpage

\section{Extended Experiments and Results}

\textbf{Training and Test Setup} Descriptions of the training problems used by us and their parameters can be found in Table.\,\ref{table:complete_data_training}. Test problems and parameters along with complete information for the solution times, costs, and their standard deviations for 10 runs are available in tabular format in Tables\,\ref{table:complete_data_keva}, \ref{table:complete_data_rover}, \ref{table:complete_data_schedule}, \ref{table:complete_data_delicate_can}, and \ref{table:complete_data_gripper}. Note that for the Keva domain, the standard deviations for the costs incurred are 0. This is accurate since the only source of stochasticity in Keva is a \emph{human place} action that determines where the human places a plank which is one of two locations. As a result, Keva policies are deterministic in execution since the human always places a plank and all other actions are deterministic. It is interesting that despite this simplistic domain, the baselines are unable to reasonably converge within the timeout. We also present an extended version of Fig.\,\ref{fig:results} of the main paper that includes results for a larger suite a test problems for a better view of our overall results. These results are reported in Fig.\,\ref{fig:results_supplemental}. Note that the solution times on the left y-axis of these plots are shown in log scale. We omitted the problems with smaller IDs, mainly whose solutions times (in log scale) were not visible for both our as well as baseline approaches, for better visualization.

\textbf{Hyperparameters} We used $\epsilon=10^{-5}$ as the value for determining whether an algorithm has converged to an $\epsilon$-consistent policy. We set the total number of trials for all algorithms to $\infty$. As a result, each algorithm would only return once it has converged or if the time limit has been exceeded. For Soft-FLARES, we used $t=4$ which controls the horizon of the sub-tree that is checked for being $\epsilon$-consistent during the labeling procedure. Our distance metric is the \emph{step} function which simply counts the depth until the horizon is exceeded. For the selective sampling procedure, we used the \emph{logistic} sampler configured with $\alpha=0.1$ and $\beta=0.9$.

\begin{figure}[h]
\centering
\includegraphics[angle=0, scale=0.05, width=\linewidth]{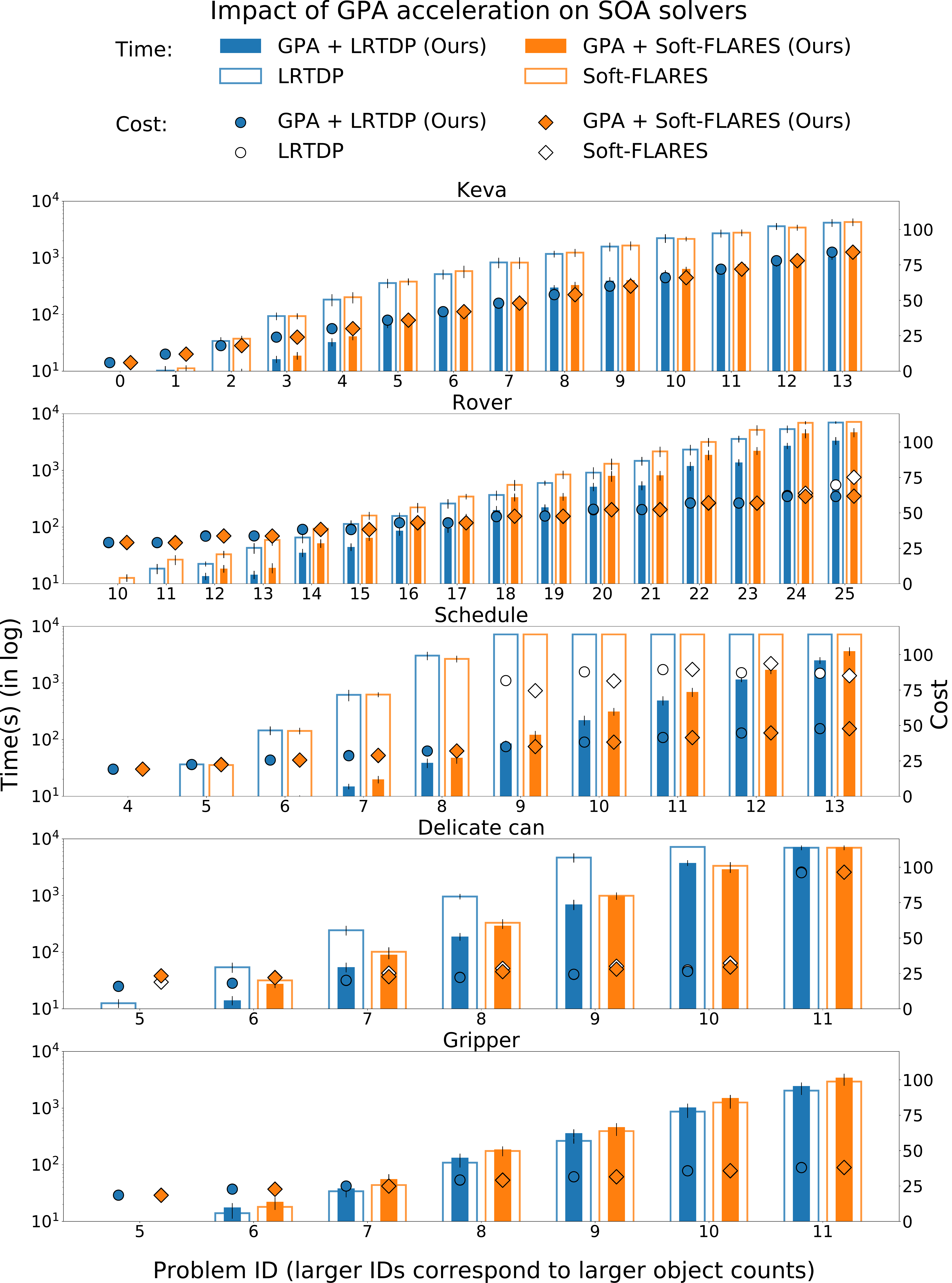}
\caption{Impact of learned GPAs on solver performance of all test problems (lower values are better). Left y-axes and bars show solution times (in log scale), and right y-axes and points show cost incurred by the policy computed by our approach and baseline SOA solvers (LRTDP and Soft-FLARES).
We use the same SSP solver as the corresponding baseline in our approach. Error bars for solution times indicate 1 standard deviation (SD) averaged across 10 runs.}
\label{fig:results_supplemental}
\end{figure}

\clearpage
\begin{landscape}

\begin{table}[t]
    \centering
    \scriptsize
    \caption{\small Our training setup for all domains. ID refers to the problem ID in the training set. The other columns refer to the parameters passed to the problem generator for generating the problem. Entries marked $-$ indicate that there was no such problem in the training set.}
    \begin{tabular}{rccccc}
    \toprule
    & \multicolumn{5}{c}{$\theta$} \\
    \cmidrule{2-6}

    ID & Keva$(P, h)$ & Rover$(r, w, s, o)$ & Schedule$(C, p)$ & Delicate Can$(c)$ & Gripper$(b)$ \\
    \midrule
    0 & $(2, 1)$ & $(1, 3, 1, 2)$ & $(1, 2)$ & $(2)$ & $(1)$ \\
    1 & $(4, 2)$ & $(1, 4, 1, 2)$ & $(1, 3)$ & $(3)$ & $(1)$ \\
    2 & $(6, 3)$ & $(1, 3, 2, 2)$ & $(1, 4)$ & $(4)$ & $(2)$ \\
    3 & $(8, 4)$ & $(1, 4, 2, 2)$ & $-$ & $(5)$ & $(2)$ \\
    4 & $(10, 5)$ & $(1, 3, 3, 2)$ & $-$ & $(6)$ & $(3)$ \\
    5 & $(12, 6)$ & $(1, 4, 3, 2)$ & $-$ & $-$ & $(3)$ \\
    6 & $-$ & $(1, 3, 4, 2)$ & $-$ & $-$ & $(4)$ \\
    7 & $-$ & $(1, 4, 4, 2)$ & $-$ & $-$ & $(4)$ \\
    8 & $-$ & $(1, 3, 5, 2)$ & $-$ & $-$ & $(5)$ \\
    9 & $-$ & $(1, 4, 5, 2)$ & $-$ & $-$ & $(5)$ \\

    \bottomrule
    \end{tabular}
    
    \label{table:complete_data_training}
\end{table}


\begin{table}[t]
    \centering
    \scriptsize
    \caption{\small Our test setup for the Keva$(P, h)$ domain (lower values better). ID refers to the problem ID in the test set. $\theta$ refers to the parameters passed to the problem generator for generating the problem. Times indicate the seconds required to find a policy. Similarly, costs are reported as average costs obtained by executing the computed policy for $100$ trials. We ran our experiments using a different random seed for 10 different runs and report average metrics up to one standard deviation. Better metrics are at least 5\% better and are indicated using bold font.}
    \begin{tabular}{rllllllllllll}
    \toprule
     &  & \multicolumn{2}{c}{Time$(x \equiv \text{LRTDP})$} & & \multicolumn{2}{c}{Time$(x \equiv \text{Soft-FLARES})$} & & \multicolumn{2}{c}{Cost$(x \equiv \text{LRTDP})$} & & \multicolumn{2}{c}{Cost$(x \equiv \text{Soft-FLARES})$} \\
    \cmidrule{3-4} \cmidrule{6-7} \cmidrule{9-10} \cmidrule{12-13}
    \multicolumn{1}{c}{ID} & \multicolumn{1}{c}{$\theta$}  & \multicolumn{1}{c}{$x$} & \multicolumn{1}{c}{Ours + $x$} & & \multicolumn{1}{c}{$x$} & \multicolumn{1}{c}{Ours + $x$} & & \multicolumn{1}{c}{$x$} & \multicolumn{1}{c}{Ours + $x$} & & \multicolumn{1}{c}{$x$} & \multicolumn{1}{c}{Ours + $x$} \\
    \midrule

0 & $(29, 1)$ & 2.45 $\pm 0.32$ & \textbf{1.44} $\pm 0.17$ & & 2.65 $\pm 0.37$ & \textbf{1.47} $\pm 0.27$ & & 6.00 $\pm 0.00$ & 6.00 $\pm 0.00$ & & 6.00 $\pm 0.00$ & 6.00 $\pm 0.00$\\
1 & $(29, 2)$ & 10.23 $\pm 2.05$ & \textbf{4.15} $\pm 0.68$ & & 11.22 $\pm 1.45$ & \textbf{4.63} $\pm 0.58$ & & 12.00 $\pm 0.00$ & 12.00 $\pm 0.00$ & & 12.00 $\pm 0.00$ & 12.00 $\pm 0.00$\\
2 & $(29, 3)$ & 34.05 $\pm 5.53$ & \textbf{8.14} $\pm 0.92$ & & 37.33 $\pm 4.71$ & \textbf{9.55} $\pm 1.45$ & & 18.00 $\pm 0.00$ & 18.00 $\pm 0.00$ & & 18.00 $\pm 0.00$ & 18.00 $\pm 0.00$\\
3 & $(29, 4)$ & 93.78 $\pm 14.53$ & \textbf{16.26} $\pm 2.34$ & & 93.39 $\pm 11.30$ & \textbf{18.80} $\pm 2.84$ & & 24.00 $\pm 0.00$ & 24.00 $\pm 0.00$ & & 24.00 $\pm 0.00$ & 24.00 $\pm 0.00$\\
4 & $(29, 5)$ & 183.57 $\pm 43.80$ & \textbf{32.61} $\pm 5.40$ & & 201.99 $\pm 45.05$ & \textbf{41.19} $\pm 6.05$ & & 30.00 $\pm 0.00$ & 30.00 $\pm 0.00$ & & 30.00 $\pm 0.00$ & 30.00 $\pm 0.00$\\
5 & $(29, 6)$ & 358.98 $\pm 67.32$ & \textbf{68.73} $\pm 11.04$ & & 379.84 $\pm 51.43$ & \textbf{74.47} $\pm 9.10$ & & 36.00 $\pm 0.00$ & 36.00 $\pm 0.00$ & & 36.00 $\pm 0.00$ & 36.00 $\pm 0.00$\\
6 & $(29, 7)$ & 515.85 $\pm 101.56$ & \textbf{115.28} $\pm 11.90$ & & 583.52 $\pm 143.15$ & \textbf{120.30} $\pm 17.58$ & & 42.00 $\pm 0.00$ & 42.00 $\pm 0.00$ & & 42.00 $\pm 0.00$ & 42.00 $\pm 0.00$\\
7 & $(29, 8)$ & 829.19 $\pm 174.46$ & \textbf{174.88} $\pm 17.36$ & & 825.84 $\pm 192.54$ & \textbf{191.44} $\pm 27.16$ & & 48.00 $\pm 0.00$ & 48.00 $\pm 0.00$ & & 48.00 $\pm 0.00$ & 48.00 $\pm 0.00$\\
8 & $(29, 9)$ & 1174.79 $\pm 160.45$ & \textbf{298.61} $\pm 30.82$ & & 1236.91 $\pm 200.02$ & \textbf{331.95} $\pm 45.28$ & & 54.00 $\pm 0.00$ & 54.00 $\pm 0.00$ & & 54.00 $\pm 0.00$ & 54.00 $\pm 0.00$\\
9 & $(29, 10)$ & 1578.38 $\pm 279.31$ & \textbf{394.71} $\pm 62.16$ & & 1647.76 $\pm 297.42$ & \textbf{406.45} $\pm 37.23$ & & 60.00 $\pm 0.00$ & 60.00 $\pm 0.00$ & & 60.00 $\pm 0.00$ & 60.00 $\pm 0.00$\\
10 & $(29, 11)$ & 2223.43 $\pm 390.05$ & \textbf{544.26} $\pm 58.84$ & & 2158.91 $\pm 199.54$ & \textbf{639.60} $\pm 54.35$ & & 66.00 $\pm 0.00$ & 66.00 $\pm 0.00$ & & 66.00 $\pm 0.00$ & 66.00 $\pm 0.00$\\
11 & $(29, 12)$ & 2713.78 $\pm 435.39$ & \textbf{665.29} $\pm 76.02$ & & 2787.92 $\pm 378.85$ & \textbf{782.82} $\pm 74.38$ & & 72.00 $\pm 0.00$ & 72.00 $\pm 0.00$ & & 72.00 $\pm 0.00$ & 72.00 $\pm 0.00$\\
12 & $(29, 13)$ & 3606.12 $\pm 494.69$ & \textbf{815.11} $\pm 117.62$ & & 3427.93 $\pm 409.34$ & \textbf{958.17} $\pm 81.59$ & & 78.00 $\pm 0.00$ & 78.00 $\pm 0.00$ & & 78.00 $\pm 0.00$ & 78.00 $\pm 0.00$\\
13 & $(29, 14)$ & 4171.27 $\pm 645.70$ & \textbf{1042.44} $\pm 119.63$ & & 4291.64 $\pm 663.83$ & \textbf{1128.84} $\pm 183.24$ & & 84.00 $\pm 0.00$ & 84.00 $\pm 0.00$ & & 84.00 $\pm 0.00$ & 84.00 $\pm 0.00$\\

    \bottomrule
    \end{tabular}
    
    \label{table:complete_data_keva}
\end{table}


\begin{table}[t]
    \centering
    \scriptsize
    \caption{\small Our test setup for the Rover$(r, w, s, o)$ domain (lower values better). ID refers to the problem ID in the test set. $\theta$ refers to the parameters passed to the problem generator for generating the problem. Times indicate the seconds required to find a policy. Similarly, costs are reported as average costs obtained by executing the computed policy for $100$ trials. We ran our experiments using a different random seed for 10 different runs and report average metrics up to one standard deviation. Better metrics are at least 5\% better and are indicated using bold font.}
    \begin{tabular}{rllllllllllll}
    \toprule
     &  & \multicolumn{2}{c}{Time$(x \equiv \text{LRTDP})$} & & \multicolumn{2}{c}{Time$(x \equiv \text{Soft-FLARES})$} & & \multicolumn{2}{c}{Cost$(x \equiv \text{LRTDP})$} & & \multicolumn{2}{c}{Cost$(x \equiv \text{Soft-FLARES})$} \\
    \cmidrule{3-4} \cmidrule{6-7} \cmidrule{9-10} \cmidrule{12-13}
    \multicolumn{1}{c}{ID} & \multicolumn{1}{c}{$\theta$}  & \multicolumn{1}{c}{$x$} & \multicolumn{1}{c}{Ours + $x$} & & \multicolumn{1}{c}{$x$} & \multicolumn{1}{c}{Ours + $x$} & & \multicolumn{1}{c}{$x$} & \multicolumn{1}{c}{Ours + $x$} & & \multicolumn{1}{c}{$x$} & \multicolumn{1}{c}{Ours + $x$} \\
    \midrule

0 & $(1, 3, 1, 2)$ & 0.02 $\pm 0.01$ & 0.02 $\pm 0.01$ & & 0.02 $\pm 0.01$ & \textbf{0.01} $\pm 0.00$ & & 6.62 $\pm 0.08$ & 6.68 $\pm 0.08$ & & 6.69 $\pm 0.12$ & 6.68 $\pm 0.09$\\
1 & $(1, 4, 1, 2)$ & 0.02 $\pm 0.01$ & 0.02 $\pm 0.00$ & & 0.02 $\pm 0.01$ & \textbf{0.01} $\pm 0.00$ & & 6.64 $\pm 0.08$ & 6.70 $\pm 0.09$ & & 6.62 $\pm 0.12$ & 6.70 $\pm 0.12$\\
2 & $(1, 3, 2, 2)$ & 0.10 $\pm 0.03$ & \textbf{0.03} $\pm 0.01$ & & 0.09 $\pm 0.02$ & \textbf{0.03} $\pm 0.00$ & & 10.33 $\pm 0.14$ & 10.29 $\pm 0.08$ & & 10.38 $\pm 0.15$ & 10.34 $\pm 0.10$\\
3 & $(1, 4, 2, 2)$ & 0.18 $\pm 0.04$ & \textbf{0.03} $\pm 0.01$ & & 0.22 $\pm 0.02$ & \textbf{0.04} $\pm 0.01$ & & 10.34 $\pm 0.11$ & 10.30 $\pm 0.12$ & & 10.27 $\pm 0.12$ & 10.26 $\pm 0.16$\\
4 & $(1, 3, 3, 2)$ & 0.43 $\pm 0.06$ & \textbf{0.18} $\pm 0.03$ & & 0.46 $\pm 0.07$ & \textbf{0.20} $\pm 0.04$ & & 15.03 $\pm 0.19$ & 15.01 $\pm 0.16$ & & 14.94 $\pm 0.23$ & 15.09 $\pm 0.11$\\
5 & $(1, 4, 3, 2)$ & 0.80 $\pm 0.14$ & \textbf{0.10} $\pm 0.02$ & & 0.96 $\pm 0.24$ & \textbf{0.09} $\pm 0.01$ & & 15.00 $\pm 0.23$ & 15.01 $\pm 0.23$ & & 14.95 $\pm 0.17$ & 14.95 $\pm 0.19$\\
6 & $(1, 3, 4, 2)$ & 1.08 $\pm 0.12$ & \textbf{0.54} $\pm 0.10$ & & 1.66 $\pm 0.39$ & \textbf{0.68} $\pm 0.11$ & & 19.76 $\pm 0.21$ & 19.65 $\pm 0.31$ & & 19.68 $\pm 0.15$ & 19.62 $\pm 0.17$\\
7 & $(1, 4, 4, 2)$ & 2.33 $\pm 0.43$ & \textbf{0.54} $\pm 0.08$ & & 3.22 $\pm 0.69$ & \textbf{0.70} $\pm 0.14$ & & 19.63 $\pm 0.13$ & 19.63 $\pm 0.18$ & & 19.70 $\pm 0.16$ & 19.73 $\pm 0.20$\\
8 & $(1, 3, 5, 2)$ & 3.54 $\pm 0.59$ & \textbf{1.61} $\pm 0.23$ & & 4.25 $\pm 0.84$ & \textbf{2.42} $\pm 0.51$ & & 24.36 $\pm 0.17$ & 24.35 $\pm 0.30$ & & 24.34 $\pm 0.22$ & 24.46 $\pm 0.09$\\
9 & $(1, 4, 5, 2)$ & 7.78 $\pm 1.62$ & \textbf{1.82} $\pm 0.28$ & & 9.57 $\pm 1.69$ & \textbf{2.26} $\pm 0.42$ & & 24.41 $\pm 0.27$ & 24.38 $\pm 0.16$ & & 24.32 $\pm 0.34$ & 24.23 $\pm 0.32$\\
10 & $(1, 3, 6, 2)$ & 8.77 $\pm 1.08$ & \textbf{4.91} $\pm 1.04$ & & 12.67 $\pm 1.91$ & \textbf{6.40} $\pm 0.93$ & & 28.95 $\pm 0.26$ & 29.13 $\pm 0.26$ & & 28.98 $\pm 0.25$ & 29.14 $\pm 0.31$\\
11 & $(1, 4, 6, 2)$ & 18.50 $\pm 3.76$ & \textbf{4.31} $\pm 0.77$ & & 26.62 $\pm 5.28$ & \textbf{5.28} $\pm 0.56$ & & 28.96 $\pm 0.29$ & 29.07 $\pm 0.18$ & & 29.09 $\pm 0.18$ & 28.86 $\pm 0.25$\\
12 & $(1, 3, 7, 2)$ & 22.42 $\pm 2.26$ & \textbf{13.60} $\pm 1.98$ & & 33.11 $\pm 4.79$ & \textbf{18.49} $\pm 2.85$ & & 33.80 $\pm 0.29$ & 33.54 $\pm 0.20$ & & 33.74 $\pm 0.26$ & 33.71 $\pm 0.24$\\
13 & $(1, 4, 7, 2)$ & 42.96 $\pm 9.15$ & \textbf{14.47} $\pm 2.40$ & & 59.22 $\pm 12.54$ & \textbf{19.09} $\pm 3.84$ & & 33.70 $\pm 0.34$ & 33.75 $\pm 0.22$ & & 33.63 $\pm 0.21$ & 33.67 $\pm 0.28$\\
14 & $(1, 3, 8, 2)$ & 65.32 $\pm 13.63$ & \textbf{35.25} $\pm 5.75$ & & 93.44 $\pm 13.87$ & \textbf{51.72} $\pm 8.90$ & & 38.30 $\pm 0.26$ & 38.39 $\pm 0.24$ & & 38.28 $\pm 0.27$ & 38.36 $\pm 0.35$\\
15 & $(1, 4, 8, 2)$ & 113.36 $\pm 17.49$ & \textbf{44.60} $\pm 6.77$ & & 159.83 $\pm 25.30$ & \textbf{64.87} $\pm 7.88$ & & 38.43 $\pm 0.17$ & 38.18 $\pm 0.22$ & & 38.33 $\pm 0.29$ & 38.24 $\pm 0.36$\\
16 & $(1, 3, 9, 2)$ & 156.88 $\pm 23.55$ & \textbf{86.26} $\pm 16.11$ & & 222.42 $\pm 43.47$ & \textbf{122.44} $\pm 21.53$ & & 43.02 $\pm 0.40$ & 42.98 $\pm 0.45$ & & 42.87 $\pm 0.26$ & 43.11 $\pm 0.29$\\
17 & $(1, 4, 9, 2)$ & 260.78 $\pm 49.61$ & \textbf{95.92} $\pm 16.68$ & & 345.68 $\pm 38.97$ & \textbf{142.71} $\pm 27.75$ & & 43.05 $\pm 0.19$ & 42.94 $\pm 0.31$ & & 42.98 $\pm 0.36$ & 43.00 $\pm 0.24$\\
18 & $(1, 3, 10, 2)$ & 367.77 $\pm 71.46$ & \textbf{199.80} $\pm 35.50$ & & 555.65 $\pm 121.90$ & \textbf{337.10} $\pm 52.84$ & & 47.65 $\pm 0.34$ & 47.36 $\pm 0.16$ & & 47.62 $\pm 0.30$ & 47.67 $\pm 0.42$\\
19 & $(1, 4, 10, 2)$ & 599.74 $\pm 60.23$ & \textbf{223.38} $\pm 26.79$ & & 848.57 $\pm 141.78$ & \textbf{345.02} $\pm 53.64$ & & 47.85 $\pm 0.39$ & 47.78 $\pm 0.27$ & & 47.49 $\pm 0.28$ & 47.70 $\pm 0.19$\\
20 & $(1, 3, 11, 2)$ & 914.39 $\pm 217.86$ & \textbf{515.34} $\pm 85.61$ & & 1312.81 $\pm 302.78$ & \textbf{800.94} $\pm 169.25$ & & 52.17 $\pm 0.22$ & 52.58 $\pm 0.35$ & & 52.29 $\pm 0.32$ & 52.21 $\pm 0.29$\\
21 & $(1, 4, 11, 2)$ & 1472.73 $\pm 254.76$ & \textbf{543.15} $\pm 97.44$ & & 2168.73 $\pm 450.03$ & \textbf{819.44} $\pm 159.47$ & & 52.36 $\pm 0.27$ & 52.35 $\pm 0.41$ & & 52.37 $\pm 0.46$ & 52.26 $\pm 0.30$\\
22 & $(1, 3, 12, 2)$ & 2336.52 $\pm 492.69$ & \textbf{1195.16} $\pm 213.55$ & & 3171.35 $\pm 554.37$ & \textbf{1885.02} $\pm 375.15$ & & 57.00 $\pm 0.35$ & 56.94 $\pm 0.26$ & & 56.85 $\pm 0.36$ & 57.12 $\pm 0.62$\\
23 & $(1, 4, 12, 2)$ & 3593.42 $\pm 487.19$ & \textbf{1385.19} $\pm 187.00$ & & 5196.66 $\pm 1063.94$ & \textbf{2224.24} $\pm 351.60$ & & 56.93 $\pm 0.24$ & 56.94 $\pm 0.24$ & & 56.80 $\pm 0.27$ & 56.94 $\pm 0.35$\\
24 & $(1, 3, 13, 2)$ & 5366.69 $\pm 807.68$ & \textbf{2721.86} $\pm 337.95$ & & 6933.16 $\pm 432.32$ & \textbf{4505.52} $\pm 811.02$ & & 62.26 $\pm 1.69$ & 61.63 $\pm 0.39$ & & 63.83 $\pm 2.43$ & 61.67 $\pm 0.29$\\
25 & $(1, 4, 13, 2)$ & 6997.37 $\pm 338.99$ & \textbf{3349.16} $\pm 517.48$ & & 7200.00 $\pm 0.00$ & \textbf{4710.61} $\pm 855.90$ & & 69.72 $\pm 11.30$ & \textbf{61.60} $\pm 0.46$ & & 75.14 $\pm 13.01$ & \textbf{61.83} $\pm 0.37$\\

    \bottomrule
    \end{tabular}
    
    \label{table:complete_data_rover}
\end{table}


\begin{table}[t]
    \centering
    \scriptsize
    \caption{\small Our test setup for the Schedule$(C, p)$ domain (lower values better). ID refers to the problem ID in the test set. $\theta$ refers to the parameters passed to the problem generator for generating the problem. Times indicate the seconds required to find a policy. Similarly, costs are reported as average costs obtained by executing the computed policy for $100$ trials. We ran our experiments using a different random seed for 10 different runs and report average metrics up to one standard deviation. Better metrics are at least 5\% better and are indicated using bold font.}
    \begin{tabular}{rllllllllllll}
    \toprule
     &  & \multicolumn{2}{c}{Time$(x \equiv \text{LRTDP})$} & & \multicolumn{2}{c}{Time$(x \equiv \text{Soft-FLARES})$} & & \multicolumn{2}{c}{Cost$(x \equiv \text{LRTDP})$} & & \multicolumn{2}{c}{Cost$(x \equiv \text{Soft-FLARES})$} \\
    \cmidrule{3-4} \cmidrule{6-7} \cmidrule{9-10} \cmidrule{12-13}
    \multicolumn{1}{c}{ID} & \multicolumn{1}{c}{$\theta$}  & \multicolumn{1}{c}{$x$} & \multicolumn{1}{c}{Ours + $x$} & & \multicolumn{1}{c}{$x$} & \multicolumn{1}{c}{Ours + $x$} & & \multicolumn{1}{c}{$x$} & \multicolumn{1}{c}{Ours + $x$} & & \multicolumn{1}{c}{$x$} & \multicolumn{1}{c}{Ours + $x$} \\
    \midrule

0 & $(1, 2)$ & 0.02 $\pm 0.01$ & 0.02 $\pm 0.01$ & & 0.02 $\pm 0.01$ & 0.02 $\pm 0.01$ & & 6.41 $\pm 0.13$ & 6.32 $\pm 0.13$ & & 6.35 $\pm 0.08$ & 6.45 $\pm 0.15$\\
1 & $(1, 3)$ & 0.08 $\pm 0.02$ & \textbf{0.06} $\pm 0.02$ & & 0.07 $\pm 0.02$ & 0.07 $\pm 0.02$ & & 9.53 $\pm 0.12$ & 9.50 $\pm 0.09$ & & 9.60 $\pm 0.13$ & 9.52 $\pm 0.12$\\
2 & $(1, 4)$ & 0.32 $\pm 0.05$ & \textbf{0.16} $\pm 0.03$ & & 0.33 $\pm 0.05$ & \textbf{0.19} $\pm 0.03$ & & 12.72 $\pm 0.14$ & 12.73 $\pm 0.14$ & & 12.81 $\pm 0.15$ & 12.77 $\pm 0.15$\\
3 & $(1, 5)$ & 1.58 $\pm 0.25$ & \textbf{0.41} $\pm 0.06$ & & 1.60 $\pm 0.34$ & \textbf{0.46} $\pm 0.09$ & & 16.00 $\pm 0.14$ & 15.97 $\pm 0.17$ & & 15.89 $\pm 0.11$ & 15.95 $\pm 0.19$\\
4 & $(1, 6)$ & 6.45 $\pm 0.77$ & \textbf{1.02} $\pm 0.22$ & & 7.36 $\pm 1.28$ & \textbf{1.21} $\pm 0.20$ & & 19.06 $\pm 0.16$ & 19.16 $\pm 0.16$ & & 19.17 $\pm 0.24$ & 19.14 $\pm 0.14$\\
5 & $(1, 7)$ & 36.46 $\pm 7.19$ & \textbf{2.46} $\pm 0.56$ & & 35.61 $\pm 6.35$ & \textbf{3.10} $\pm 0.59$ & & 22.37 $\pm 0.19$ & 22.45 $\pm 0.19$ & & 22.45 $\pm 0.25$ & 22.28 $\pm 0.24$\\
6 & $(1, 8)$ & 145.33 $\pm 24.97$ & \textbf{6.58} $\pm 1.25$ & & 142.70 $\pm 18.86$ & \textbf{8.42} $\pm 1.93$ & & 25.57 $\pm 0.18$ & 25.56 $\pm 0.09$ & & 25.52 $\pm 0.23$ & 25.59 $\pm 0.26$\\
7 & $(1, 9)$ & 616.36 $\pm 140.89$ & \textbf{14.92} $\pm 1.65$ & & 622.93 $\pm 61.96$ & \textbf{19.85} $\pm 3.00$ & & 28.78 $\pm 0.18$ & 28.61 $\pm 0.16$ & & 28.73 $\pm 0.25$ & 28.88 $\pm 0.18$\\
8 & $(1, 10)$ & 3036.01 $\pm 507.41$ & \textbf{38.89} $\pm 7.41$ & & 2662.96 $\pm 361.97$ & \textbf{48.08} $\pm 10.54$ & & 31.95 $\pm 0.08$ & 31.95 $\pm 0.23$ & & 31.96 $\pm 0.20$ & 32.01 $\pm 0.21$\\
9 & $(1, 11)$ & 7200.00 $\pm 0.00$ & \textbf{85.99} $\pm 9.94$ & & 7200.00 $\pm 0.00$ & \textbf{122.28} $\pm 19.92$ & & 81.68 $\pm 16.65$ & \textbf{35.06} $\pm 0.29$ & & 74.56 $\pm 19.98$ & \textbf{35.10} $\pm 0.28$\\
10 & $(1, 12)$ & 7200.00 $\pm 0.00$ & \textbf{220.64} $\pm 43.34$ & & 7200.00 $\pm 0.00$ & \textbf{313.88} $\pm 48.81$ & & 87.85 $\pm 11.06$ & \textbf{38.26} $\pm 0.23$ & & 81.42 $\pm 16.54$ & \textbf{38.20} $\pm 0.25$\\
11 & $(1, 13)$ & 7200.00 $\pm 0.00$ & \textbf{490.90} $\pm 90.83$ & & 7200.00 $\pm 0.00$ & \textbf{692.75} $\pm 126.60$ & & 89.45 $\pm 12.88$ & \textbf{41.50} $\pm 0.40$ & & 89.54 $\pm 11.68$ & \textbf{41.43} $\pm 0.22$\\
12 & $(1, 14)$ & 7200.00 $\pm 0.00$ & \textbf{1153.32} $\pm 117.37$ & & 7200.00 $\pm 0.00$ & \textbf{1710.56} $\pm 278.99$ & & 87.26 $\pm 12.32$ & \textbf{44.69} $\pm 0.13$ & & 93.66 $\pm 10.10$ & \textbf{44.74} $\pm 0.27$\\
13 & $(1, 15)$ & 7200.00 $\pm 0.00$ & \textbf{2514.16} $\pm 337.59$ & & 7200.00 $\pm 0.00$ & \textbf{3639.92} $\pm 627.08$ & & 86.84 $\pm 11.41$ & \textbf{47.77} $\pm 0.17$ & & 85.14 $\pm 13.60$ & \textbf{47.74} $\pm 0.34$\\

    \bottomrule
    \end{tabular}
    
    \label{table:complete_data_schedule}
\end{table}


\begin{table}[t]
    \centering
    \scriptsize
    \caption{\small Our test setup for the Delicate Can$(c)$ domain (lower values better). ID refers to the problem ID in the test set. $\theta$ refers to the parameters passed to the problem generator for generating the problem. Times indicate the seconds required to find a policy. Similarly, costs are reported as average costs obtained by executing the computed policy for $100$ trials. We ran our experiments using a different random seed for 10 different runs and report average metrics up to one standard deviation. Better metrics are at least 5\% better and are indicated using bold font.}
    \begin{tabular}{rllllllllllll}
    \toprule
     &  & \multicolumn{2}{c}{Time$(x \equiv \text{LRTDP})$} & & \multicolumn{2}{c}{Time$(x \equiv \text{Soft-FLARES})$} & & \multicolumn{2}{c}{Cost$(x \equiv \text{LRTDP})$} & & \multicolumn{2}{c}{Cost$(x \equiv \text{Soft-FLARES})$} \\
    \cmidrule{3-4} \cmidrule{6-7} \cmidrule{9-10} \cmidrule{12-13}
    \multicolumn{1}{c}{ID} & \multicolumn{1}{c}{$\theta$}  & \multicolumn{1}{c}{$x$} & \multicolumn{1}{c}{Ours + $x$} & & \multicolumn{1}{c}{$x$} & \multicolumn{1}{c}{Ours + $x$} & & \multicolumn{1}{c}{$x$} & \multicolumn{1}{c}{Ours + $x$} & & \multicolumn{1}{c}{$x$} & \multicolumn{1}{c}{Ours + $x$} \\
    \midrule

0 & $(2)$ & 0.01 $\pm 0.01$ & 0.01 $\pm 0.00$ & & \textbf{0.00} $\pm 0.00$ & 0.01 $\pm 0.01$ & & 5.40 $\pm 0.07$ & 5.40 $\pm 0.08$ & & 5.45 $\pm 0.07$ & 5.40 $\pm 0.10$\\
1 & $(3)$ & 0.03 $\pm 0.01$ & \textbf{0.02} $\pm 0.01$ & & 0.03 $\pm 0.01$ & \textbf{0.02} $\pm 0.01$ & & 7.42 $\pm 0.12$ & 7.50 $\pm 0.15$ & & 7.53 $\pm 0.09$ & 7.48 $\pm 0.06$\\
2 & $(4)$ & 0.11 $\pm 0.03$ & \textbf{0.07} $\pm 0.02$ & & 0.12 $\pm 0.03$ & \textbf{0.09} $\pm 0.02$ & & 9.58 $\pm 0.11$ & 9.53 $\pm 0.16$ & & 9.56 $\pm 0.14$ & 9.62 $\pm 0.09$\\
3 & $(5)$ & 0.49 $\pm 0.08$ & \textbf{0.22} $\pm 0.05$ & & 0.51 $\pm 0.08$ & \textbf{0.40} $\pm 0.06$ & & 11.71 $\pm 0.20$ & 11.72 $\pm 0.12$ & & \textbf{12.48} $\pm 0.69$ & 14.22 $\pm 2.49$\\
4 & $(6)$ & 2.56 $\pm 0.55$ & \textbf{0.80} $\pm 0.13$ & & 2.14 $\pm 0.26$ & \textbf{1.79} $\pm 0.23$ & & 13.82 $\pm 0.15$ & 13.75 $\pm 0.07$ & & \textbf{15.02} $\pm 1.09$ & 25.16 $\pm 4.14$\\
5 & $(7)$ & 12.55 $\pm 2.20$ & \textbf{3.25} $\pm 0.56$ & & 8.08 $\pm 0.92$ & \textbf{6.81} $\pm 1.41$ & & 15.94 $\pm 0.09$ & 15.86 $\pm 0.08$ & & \textbf{18.87} $\pm 1.44$ & 23.32 $\pm 4.09$\\
6 & $(8)$ & 54.31 $\pm 11.04$ & \textbf{14.12} $\pm 2.60$ & & 31.86 $\pm 5.15$ & \textbf{27.65} $\pm 4.41$ & & 18.08 $\pm 0.13$ & 18.07 $\pm 0.14$ & & 22.16 $\pm 1.68$ & 22.01 $\pm 1.52$\\
7 & $(9)$ & 244.32 $\pm 47.46$ & \textbf{54.60} $\pm 10.68$ & & 101.63 $\pm 20.29$ & \textbf{90.57} $\pm 15.57$ & & 20.04 $\pm 0.13$ & 20.15 $\pm 0.14$ & & 25.18 $\pm 2.65$ & \textbf{22.70} $\pm 1.22$\\
8 & $(10)$ & 960.96 $\pm 105.78$ & \textbf{188.50} $\pm 29.21$ & & 331.06 $\pm 49.98$ & \textbf{293.54} $\pm 36.06$ & & 22.28 $\pm 0.17$ & 22.17 $\pm 0.13$ & & 28.54 $\pm 2.09$ & \textbf{26.29} $\pm 1.73$\\
9 & $(11)$ & 4691.23 $\pm 859.08$ & \textbf{698.09} $\pm 143.63$ & & 990.65 $\pm 144.48$ & 983.51 $\pm 144.08$ & & 24.30 $\pm 0.12$ & 24.41 $\pm 0.17$ & & 30.05 $\pm 1.37$ & \textbf{28.02} $\pm 1.40$\\
10 & $(12)$ & 7200.00 $\pm 0.00$ & \textbf{3769.57} $\pm 439.21$ & & 3347.02 $\pm 540.66$ & \textbf{2911.30} $\pm 397.41$ & & 27.46 $\pm 0.83$ & 26.46 $\pm 0.14$ & & 32.23 $\pm 2.35$ & \textbf{29.49} $\pm 1.47$\\
11 & $(12)$ & 6980.00 $\pm 660.00$ & 7183.58 $\pm 49.27$ & & 6980.48 $\pm 658.56$ & 7132.65 $\pm 202.04$ & & 96.62 $\pm 10.13$ & 96.12 $\pm 11.65$ & & 96.50 $\pm 10.49$ & 96.48 $\pm 10.56$\\

    \bottomrule
    \end{tabular}
    
    \label{table:complete_data_delicate_can}
\end{table}


\begin{table}[t]
    \centering
    \scriptsize
    \caption{\small Our test setup for the Gripper$(b)$ domain (lower values better). ID refers to the problem ID in the test set. $\theta$ refers to the parameters passed to the problem generator for generating the problem. Times indicate the seconds required to find a policy. Similarly, costs are reported as average costs obtained by executing the computed policy for $100$ trials. We ran our experiments using a different random seed for 10 different runs and report average metrics up to one standard deviation. Better metrics are at least 5\% better and are indicated using bold font.}
    \begin{tabular}{rllllllllllll}
    \toprule
     &  & \multicolumn{2}{c}{Time$(x \equiv \text{LRTDP})$} & & \multicolumn{2}{c}{Time$(x \equiv \text{Soft-FLARES})$} & & \multicolumn{2}{c}{Cost$(x \equiv \text{LRTDP})$} & & \multicolumn{2}{c}{Cost$(x \equiv \text{Soft-FLARES})$} \\
    \cmidrule{3-4} \cmidrule{6-7} \cmidrule{9-10} \cmidrule{12-13}
    \multicolumn{1}{c}{ID} & \multicolumn{1}{c}{$\theta$}  & \multicolumn{1}{c}{$x$} & \multicolumn{1}{c}{Ours + $x$} & & \multicolumn{1}{c}{$x$} & \multicolumn{1}{c}{Ours + $x$} & & \multicolumn{1}{c}{$x$} & \multicolumn{1}{c}{Ours + $x$} & & \multicolumn{1}{c}{$x$} & \multicolumn{1}{c}{Ours + $x$} \\
    \midrule

0 & $(1)$ & 0.01 $\pm 0.01$ & \textbf{0.00} $\pm 0.01$ & & \textbf{0.00} $\pm 0.00$ & 0.01 $\pm 0.01$ & & 3.25 $\pm 0.06$ & 3.25 $\pm 0.05$ & & 3.23 $\pm 0.06$ & 3.25 $\pm 0.04$\\
1 & $(2)$ & 0.02 $\pm 0.01$ & 0.02 $\pm 0.01$ & & 0.02 $\pm 0.01$ & 0.02 $\pm 0.01$ & & 5.48 $\pm 0.08$ & 5.53 $\pm 0.08$ & & 5.51 $\pm 0.10$ & 5.55 $\pm 0.07$\\
2 & $(3)$ & \textbf{0.10} $\pm 0.02$ & 0.12 $\pm 0.03$ & & 0.12 $\pm 0.03$ & 0.12 $\pm 0.03$ & & 9.71 $\pm 0.07$ & 9.76 $\pm 0.10$ & & 9.76 $\pm 0.07$ & 9.75 $\pm 0.09$\\
3 & $(4)$ & 0.33 $\pm 0.05$ & 0.34 $\pm 0.05$ & & \textbf{0.30} $\pm 0.06$ & 0.38 $\pm 0.07$ & & 11.99 $\pm 0.10$ & 12.03 $\pm 0.16$ & & 12.00 $\pm 0.08$ & 12.03 $\pm 0.10$\\
4 & $(5)$ & \textbf{1.36} $\pm 0.23$ & 1.62 $\pm 0.28$ & & \textbf{1.82} $\pm 0.35$ & 2.21 $\pm 0.47$ & & 16.27 $\pm 0.09$ & 16.23 $\pm 0.12$ & & 16.22 $\pm 0.06$ & 16.26 $\pm 0.15$\\
5 & $(6)$ & \textbf{3.38} $\pm 0.39$ & 5.17 $\pm 1.65$ & & \textbf{4.64} $\pm 0.80$ & 5.16 $\pm 0.64$ & & 18.51 $\pm 0.21$ & 18.54 $\pm 0.14$ & & 18.53 $\pm 0.12$ & 18.44 $\pm 0.07$\\
6 & $(7)$ & \textbf{14.00} $\pm 2.83$ & 17.67 $\pm 3.43$ & & \textbf{18.07} $\pm 2.13$ & 22.28 $\pm 4.48$ & & 22.82 $\pm 0.08$ & 22.78 $\pm 0.10$ & & 22.80 $\pm 0.19$ & 22.73 $\pm 0.15$\\
7 & $(8)$ & \textbf{34.11} $\pm 7.52$ & 38.35 $\pm 4.68$ & & \textbf{43.88} $\pm 8.90$ & 55.78 $\pm 12.25$ & & 24.96 $\pm 0.12$ & 24.99 $\pm 0.13$ & & 24.96 $\pm 0.14$ & 24.95 $\pm 0.14$\\
8 & $(9)$ & \textbf{108.94} $\pm 19.36$ & 133.06 $\pm 24.75$ & & \textbf{175.48} $\pm 33.68$ & 186.91 $\pm 24.64$ & & 29.17 $\pm 0.10$ & 29.26 $\pm 0.20$ & & 29.23 $\pm 0.14$ & 29.08 $\pm 0.17$\\
9 & $(10)$ & \textbf{264.59} $\pm 28.50$ & 362.16 $\pm 60.65$ & & \textbf{392.91} $\pm 69.13$ & 463.30 $\pm 78.81$ & & 31.48 $\pm 0.20$ & 31.54 $\pm 0.15$ & & 31.47 $\pm 0.14$ & 31.54 $\pm 0.17$\\
10 & $(11)$ & \textbf{867.91} $\pm 195.44$ & 1028.70 $\pm 172.62$ & & \textbf{1256.52} $\pm 277.94$ & 1509.25 $\pm 195.14$ & & 35.78 $\pm 0.09$ & 35.73 $\pm 0.15$ & & 35.70 $\pm 0.15$ & 35.84 $\pm 0.16$\\
11 & $(12)$ & \textbf{2037.65} $\pm 326.63$ & 2454.41 $\pm 377.38$ & & \textbf{2941.24} $\pm 457.95$ & 3461.67 $\pm 595.74$ & & 38.07 $\pm 0.13$ & 37.98 $\pm 0.15$ & & 38.03 $\pm 0.18$ & 38.10 $\pm 0.12$\\

    \bottomrule
    \end{tabular}
    
    \label{table:complete_data_gripper}
\end{table}

\end{landscape}

\end{document}